\documentclass{article} 
\usepackage{iclr2027_conference,times}

\usepackage{amsmath,amsfonts,bm}

\def\eqref#1{equation~\ref{#1}}

\def\1{\bm{1}}

\DeclareMathAlphabet{\mathsfit}{\encodingdefault}{\sfdefault}{m}{sl}
\SetMathAlphabet{\mathsfit}{bold}{\encodingdefault}{\sfdefault}{bx}{n}

\def\sR{{\mathbb{R}}}

\usepackage{hyperref}
\usepackage{url}
\usepackage{subcaption}
\usepackage{graphicx}
\usepackage{float}
\usepackage{algorithm}
\usepackage{algorithmic}
\usepackage{amsthm}
\usepackage{booktabs}
\usepackage{multirow}
\usepackage{soul}

\newcommand{\resp}{\textcolor{black}}

\title{Missing Data Imputation under Manifold Hypothesis}

\author{Bi, Zelong \\
School of Mathematics \& Statistics\\
University of New South Wales\\
High St, Kensington, NSW 2052 \\
\texttt{zelong.bi@unsw.edu.au} \\
\And
Ibenegbu, Amuchechukwu \\
School of Mathematics \& Statistics\\
University of New South Wales\\
High St, Kensington, NSW 2052 \\
\texttt{a.ibenegbu@unsw.edu.au} \\
\And
Moka, Sarat \\
School of Mathematics \& Statistics\\
University of New South Wales\\
High St, Kensington, NSW 2052 \\
\texttt{s.moka@unsw.edu.au} \\
}

\def\sR{{\mathbb{R}}}
\newtheorem{theorem}{Theorem}
\newtheorem*{theorem*}{Theorem}

\begin{document}

\maketitle

\begin{abstract}
The manifold hypothesis posits that high-dimensional data are concentrated near a low-dimensional embedded manifold. Recent advances in mixture variational autoencoders (VAEs) provide a powerful tool for extracting such underlying structure in a faithful manner. The resulting geometric structure naturally introduces local and global relationships among variables, thereby providing a systematic way of imputing missing data. We propose a model-based imputation method that enables sampling from \( p(\bm{x}_{\mathrm{mis}} \mid \bm{x}_{\mathrm{obs}}) \) via a sampling-importance-resampling (SIR) procedure, which can be further augmented with a joint diffusion model in the latent space. Our method imputes missing data while respecting the underlying geometry, achieves competitive performance compared to state-of-the-art procedures, quantifies uncertainty in the imputations, and is model-based, thereby enabling on-the-fly imputation without rerunning the entire procedure.
\end{abstract}

\section{Introduction}

Missing data imputation investigates systematic methods for filling missing values in a dataset \citep{rubin1976inference}, while the manifold hypothesis posits that high-dimensional data are concentrated near a low-dimensional embedded manifold \citep{meilua2024manifold}. The former seeks to make use of the structure, while the latter provides a structure; hence, the two are naturally related. Hence, it is surprising that mainstream data imputation methods have yet to make this connection. As shown in Figure~\ref{fig:01b}, state-of-the-art leading methods, such as MissForest \citep{stekhoven2012missforest}, fail to respect the underlying geometry when it exists. On the other hand, manifold structure provides deterministic relationships among variables. As a simple example, if the underlying manifold is known to be a unit circle, then given the value of $x_1$, we know for sure that $x_2$ can only be $\pm\sqrt{1 - x_1^2}$. Imputation quality can be greatly improved if we make use of the geometric information, as shown in Figure~\ref{fig:01c}.

\begin{figure}[H]
    \centering
    \begin{subfigure}[b]{0.3\textwidth}
        \centering
        \includegraphics[width=\textwidth]{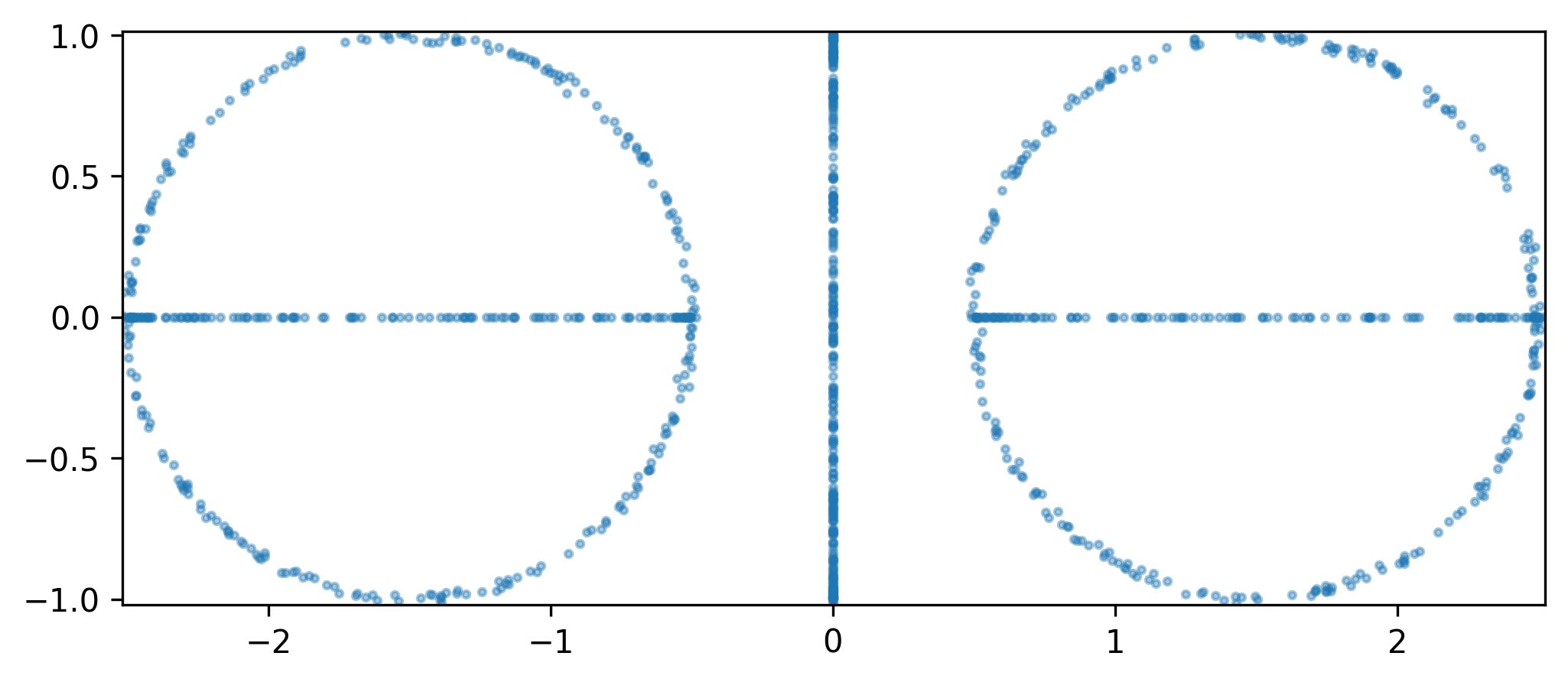}
        \caption{Data with missing values}
        \label{fig:01a}
    \end{subfigure}
    \hfill
    \begin{subfigure}[b]{0.3\textwidth}
        \centering
        \includegraphics[width=\textwidth]{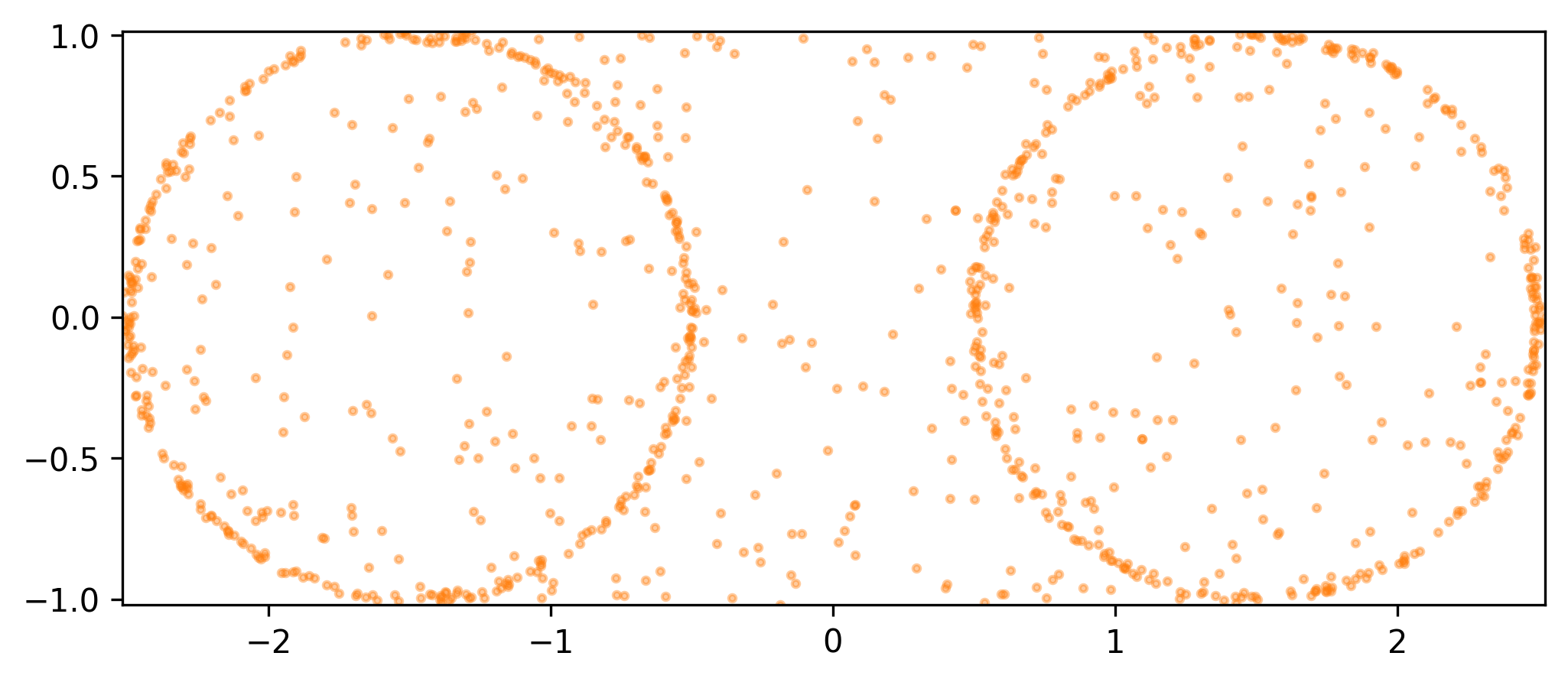}
        \caption{MissForest}
        \label{fig:01b}
    \end{subfigure}
    \hfill
    \begin{subfigure}[b]{0.3\textwidth}
        \centering
        \includegraphics[width=\textwidth]{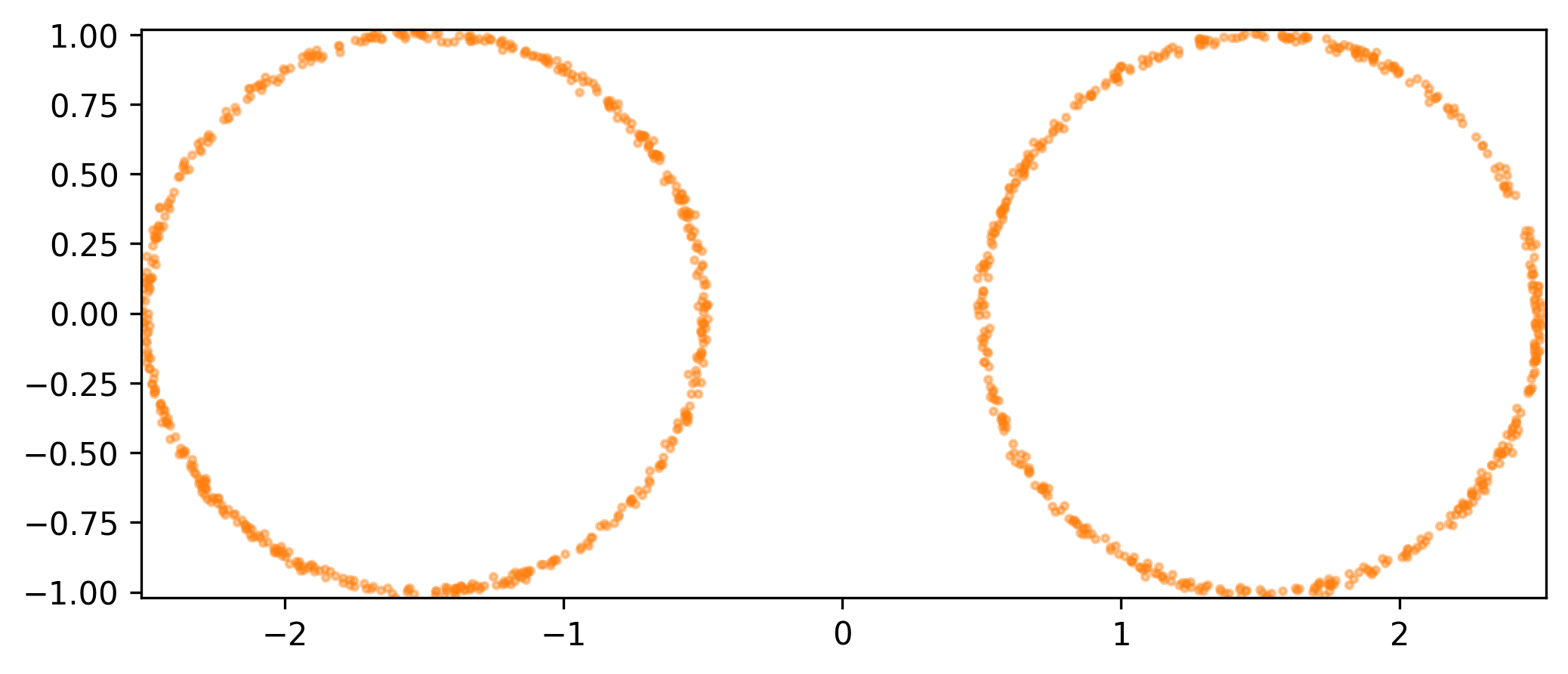}
        \caption{\resp{Our method}}
        \label{fig:01c}
    \end{subfigure}
    \caption{\resp{Comparison of the imputation results between MissForest and our proposed method. In Figure~\ref{fig:01a}, missing $x_2$ entries in the data points $(x_1, x_2)$ are replaced with zeros.}}
    \label{fig:01}
\end{figure}

In practice, the underlying manifold structure is unknown and must be learned, and given the observed part  $\bm{x}_\mathrm{obs}$ of a data point, a systematic way of imputing the missing part $\bm{x}_\mathrm{mis}$ is needed, which ideally should respect both the underlying geometry and the data distribution. The former problem is known as \textit{manifold learning}, and mixture models of VAEs by \cite{alberti2024manifold} provide the most recent and promising method to recover the underlying geometric structure, given that the manifold hypothesis holds to a reasonable extent\resp{, i.e., a low-dimensional manifold is sufficient to capture the essential structure of the dataset.} In this article, we solve the second problem by making use of the captured manifold structure. In particular, we draw directly from \( p(\bm{x}_{\mathrm{mis}} \mid \bm{x}_{\mathrm{obs}}) \) via a Bayesian procedure, yielding imputations that are both geometry- and distribution-aware. \resp{This procedure can be further enhanced by introducing a joint diffusion process which accommodates both continuous variables and discrete variables}, enabling the generation of high-quality samples that preserve the underlying geometric and distributional characteristics.

Our main contributions are as follows:
\begin{enumerate}
    \item We propose an imputation procedure that samples directly from \( p(\bm{x}_{\mathrm{mis}} \mid \bm{x}_{\mathrm{obs}}) \) with the help of the sampling-importance-resampling (SIR) procedure \citep{rubin1987calculation} that leverages the learned manifold structure. \resp{This yields a model-based imputation approach that respects both the geometric and distributional structure of the data, enables uncertainty quantification of the fills, and supports on-the-fly imputation without rerunning the entire procedure.}

    \item \resp{We also develop a joint diffusion process in the low-dimensional coordinate space that simultaneously updates the coordinates and membership labels associated with different regions of the underlying manifold. This enables efficient and high-quality data generation which preserves the underlying geometry and data distribution, thereby improving the performance of our imputation method.}

    \item Experiments show that our proposed method imputes missing data while respecting the underlying geometry and achieves competitive performance compared to state-of-the-art procedures, evaluated using root mean squared error (RMSE), Wasserstein distance, as well as downstream task performance.
\end{enumerate}

For the remainder of the article, Section~\ref{background} introduces relevant background knowledge for our framework, Section~\ref{design} describes the proposed procedure, experimental results are presented in Section~\ref{experiments}, Section~\ref{conclusion} concludes the article, and additional materials are provided in the appendix.

\section{Background and Related Work}\label{background}

\textbf{Missing Data Imputation}. Missing data imputation seeks to estimate and replace missing values in a dataset systematically. Classical approaches include mean imputation, principal component analysis (PCA) imputation \citep{dray2015principal}, and $k$-nearest neighbor (KNN) imputation \citep{troyanskaya2001missing}. More advanced methods include Multiple Imputation by Chained Equations (MICE) \citep{van2011mice}, its nonparametric extension MissForest \citep{stekhoven2012missforest}, and generative approaches such as Generative Adversarial Imputation Networks (GAIN) \citep{yoon2018gain} and variational autoencoder (VAE)-based methods \citep{mccoy2018variational}. All of these have provided the foundation for many more recent approaches \citep{ou2024missing, albu2025missforestpredict}. Among these, MissForest is often considered one of the strongest general-purpose methods due to its ability to model nonlinear relationships and mixed data types.

Most imputation methods are \emph{transductive}, directly filling missing values in a given dataset without learning a reusable model. In contrast, \emph{inductive} methods such as VAEs learn a predictive model that can be applied to unseen data. \resp{Transductive methods often achieve higher accuracy on a fixed dataset because they iteratively optimize the imputations using all available observations. However, when new data arrive, the entire optimization procedure must be repeated to obtain imputations for the expanded dataset.} Inductive methods overcome this challenge, usually at the cost of requiring a sufficient number of complete instances for learning a model.

\textbf{MCAR, MAR and MNAR}. Let $\bm{x} \in \sR^p$ denote a data vector and $\bm{m} \in \{0,1\}^p$ a missingness indicator, where $m_j=0$ indicates that $j$-th element $x_j$ is missing. The observed data are then $\bm{x}\odot\bm{m}$, whose components can be partitioned into the observed variables $\bm{x}_{\mathrm{obs}}$ and the missing variables $\bm{x}_{\mathrm{mis}}$. 

The three standard missingness mechanisms, Missing Completely at Random (MCAR), Missing at Random (MAR), and Missing Not at Random (MNAR), were introduced by \cite{rubin1976inference}. MCAR assumes that $\bm{x}$ and $\bm{m}$ are independent. MAR assumes for any $\bm{m}_0$, $\bm{x}_1$, and $\bm{x}_2$,
\begin{equation}
    p(\bm{m}=\bm{m}_0\mid \bm{x}=\bm{x}_1) = p(\bm{m}=\bm{m}_0\mid \bm{x}=\bm{x}_2)
\end{equation}
whenever $\bm{x}_1\odot\bm{m}_0=\bm{x}_2\odot\bm{m}_0$. Any mechanism that is neither MCAR nor MAR is classified as MNAR. The difficulty of imputation depends strongly on the missingness mechanism, with MNAR generally presenting greater challenges than MCAR and MAR. \resp{Like most other imputation methods, our proposal is developed under the MCAR assumption, relying on $p(\bm{x}\mid \bm{m}=\bm{1}) = p(\bm{x})$ to ensure that a model trained on complete cases can be applied to data with missing components, but we also extensively evaluate its robustness under the other two missingness mechanisms.}

\textbf{Manifolds and Intrinsic Dimension}. For $d \leq p$, a $d$-dimensional embedded manifold in $\sR^p$ is a smooth geometric object that locally admits a $d$-dimensional coordinate representation everywhere. In particular, it can be covered by a collection of possibly overlapping regions called \textit{charts}, each of which can be mapped smoothly to $\sR^d$. The manifold hypothesis posits that data in $\sR^p$ lie on or near such a manifold with dimension $d \ll p$, where $d$ is referred to as the \textit{intrinsic dimension}.

Estimating intrinsic dimension is a challenging problem that has attracted many approaches, notable methods include maximum likelihood estimators (MLE, \cite{levina2004maximum}), the two-nearest neighbor estimator (TwoNN, \cite{facco2017estimating}), and curvature-aware principal component analysis (CA-PCA, \cite{gilbert2025pca}).

\textbf{Manifold Learning via Mixture VAEs}.
Classical manifold learning methods include Isomap \citep{tenenbaum2000global}, diffusion maps \citep{coifman2006diffusion}, and UMAP \citep{mcinnes2018umap}, which are typically designed for visualization rather than faithful recovery of the underlying geometry. More recently, \cite{alberti2024manifold} proposed a mixture of VAEs, where each component can be interpreted as a chart, making the model naturally aligned with the manifold structure and more faithful in capturing the underlying geometry.

Let $\bm{x} \in \sR^p$ denote the original data variable, latent variables $\bm{z} \in \mathbb{R}^d$ and $c \in \{1, \dots, C\}$ are introduced, where $\bm{z}$ corresponds to local coordinates and $c$ indexes the charts, together with some simple prior $p(\bm{z})$. This yields the encoder--decoder factorization
\begin{equation}
p(\bm{z}, c \mid \bm{x}) = p(\bm{z} \mid c, \bm{x})p(c \mid \bm{x}), 
\qquad
p(\bm{x}, \bm{z}, c) = p(\bm{x} \mid \bm{z}, c)p(c \mid \bm{z})p(\bm{z}).
\end{equation}

As in standard VAEs, the conditional distributions $p(\bm{z} \mid c, \bm{x})$ and $p(\bm{x} \mid \bm{z}, c)$ are intractable and are therefore approximated by Gaussian distributions:
\begin{equation}\label{eqn-conds}
p(\bm{z} \mid c, \bm{x})
\ \dot\sim \
\mathcal{N}(E_c(\bm{x}), \sigma_z^2 I_d),
\qquad
p(\bm{x} \mid \bm{z}, c)
\ \dot\sim \
\mathcal{N}(D_c(\bm{z}), \sigma_x^2 I_p),
\end{equation}
where $E_c : \mathbb{R}^p \to \mathbb{R}^d$ and $D_c : \mathbb{R}^d \to \mathbb{R}^p$ are smooth neural network maps, with $D_c$ injective and $E_c \circ D_c$ \resp{being the identity map}, corresponding to local coordinate maps and their inverses. The loss function is defined to reflect both geometric and likelihood constraints during training, which in the end yield the learned distributions
\[
p(\bm{z} \mid c, \bm{x}), \qquad
p(\bm{x} \mid \bm{z}, c), \qquad
p(c), \qquad
p(c \mid \bm{x}).
\]

\begin{figure}[H]
    \centering

    \begin{subfigure}[b]{0.23\textwidth}
        \centering
        \includegraphics[width=\textwidth]{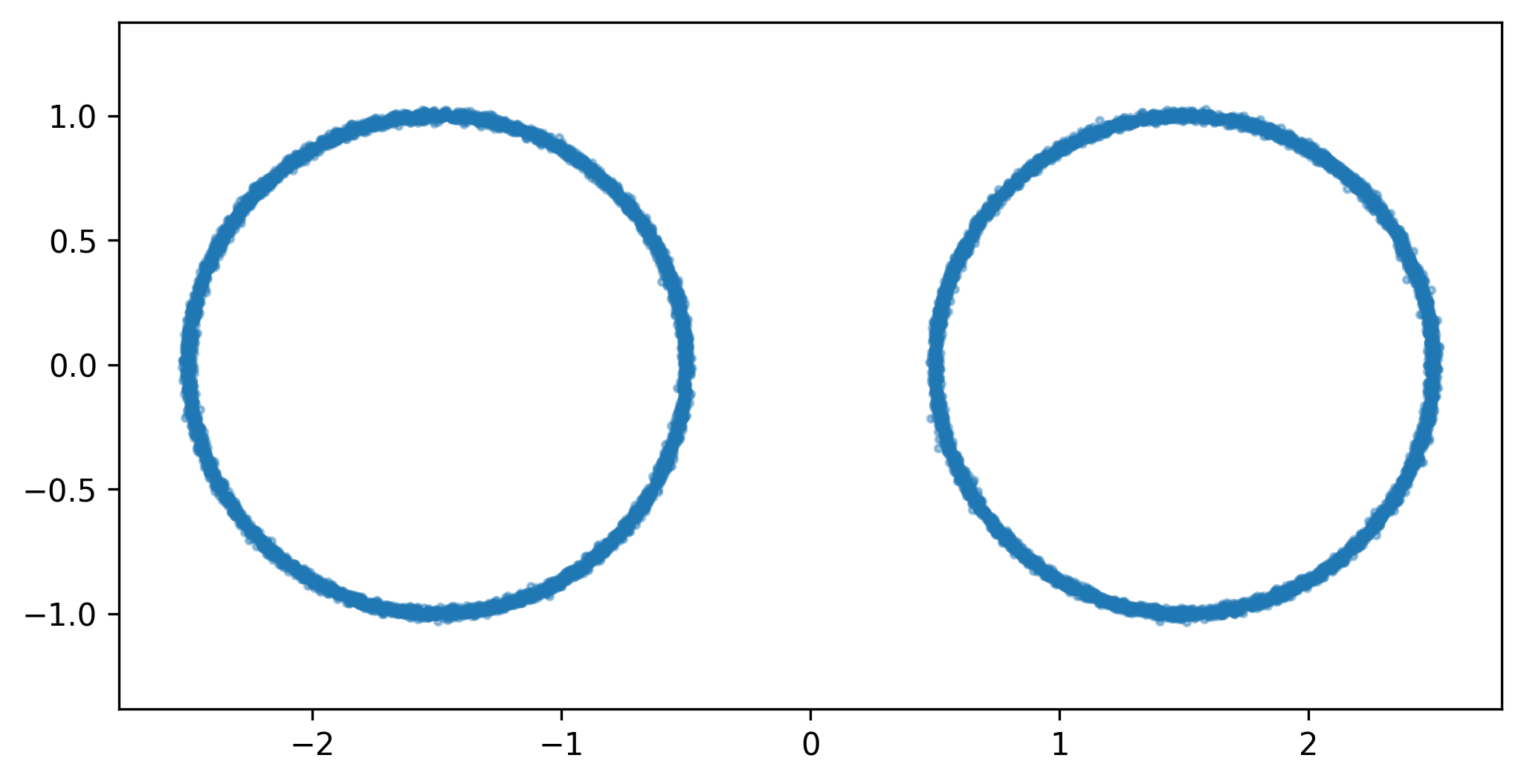}
        \caption{}
    \end{subfigure}
    \hfill
    \begin{subfigure}[b]{0.23\textwidth}
        \centering
        \includegraphics[width=\textwidth]{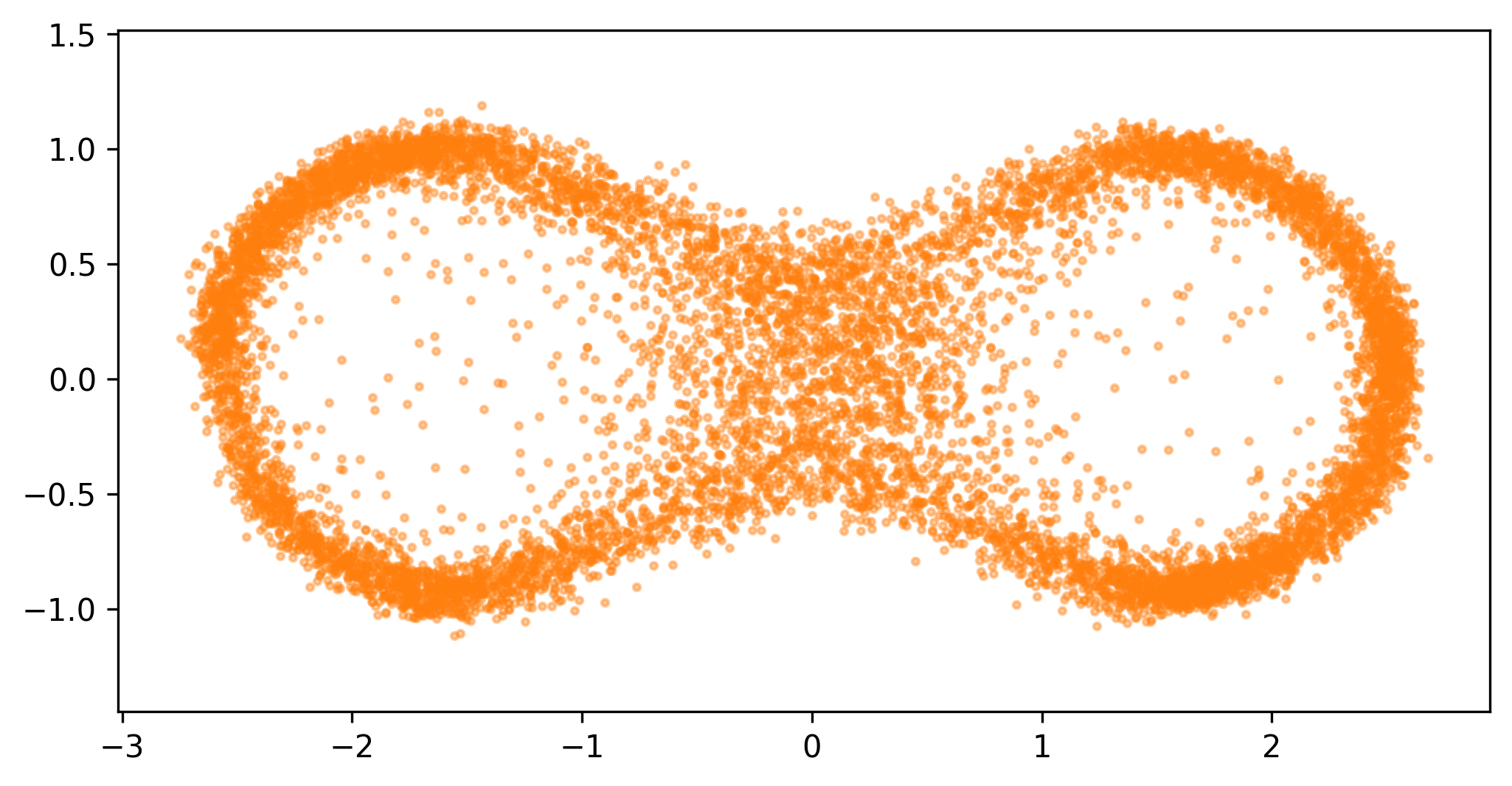}
        \caption{}
    \end{subfigure}
    \hfill
    \begin{subfigure}[b]{0.23\textwidth}
        \centering
        \includegraphics[width=\textwidth]{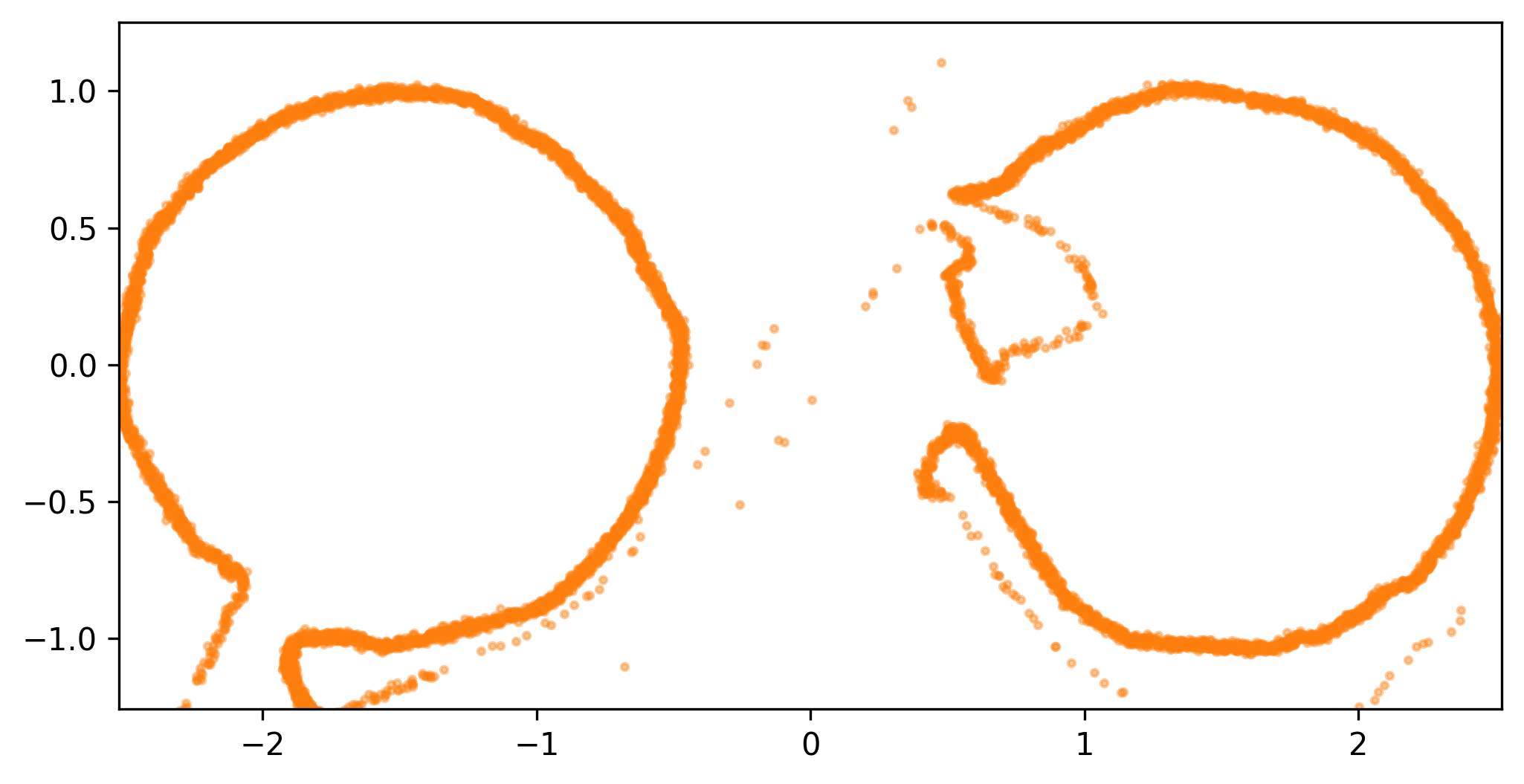}
        \caption{}
    \end{subfigure}
    \hfill
    \begin{subfigure}[b]{0.23\textwidth}
        \centering
        \includegraphics[width=\textwidth]{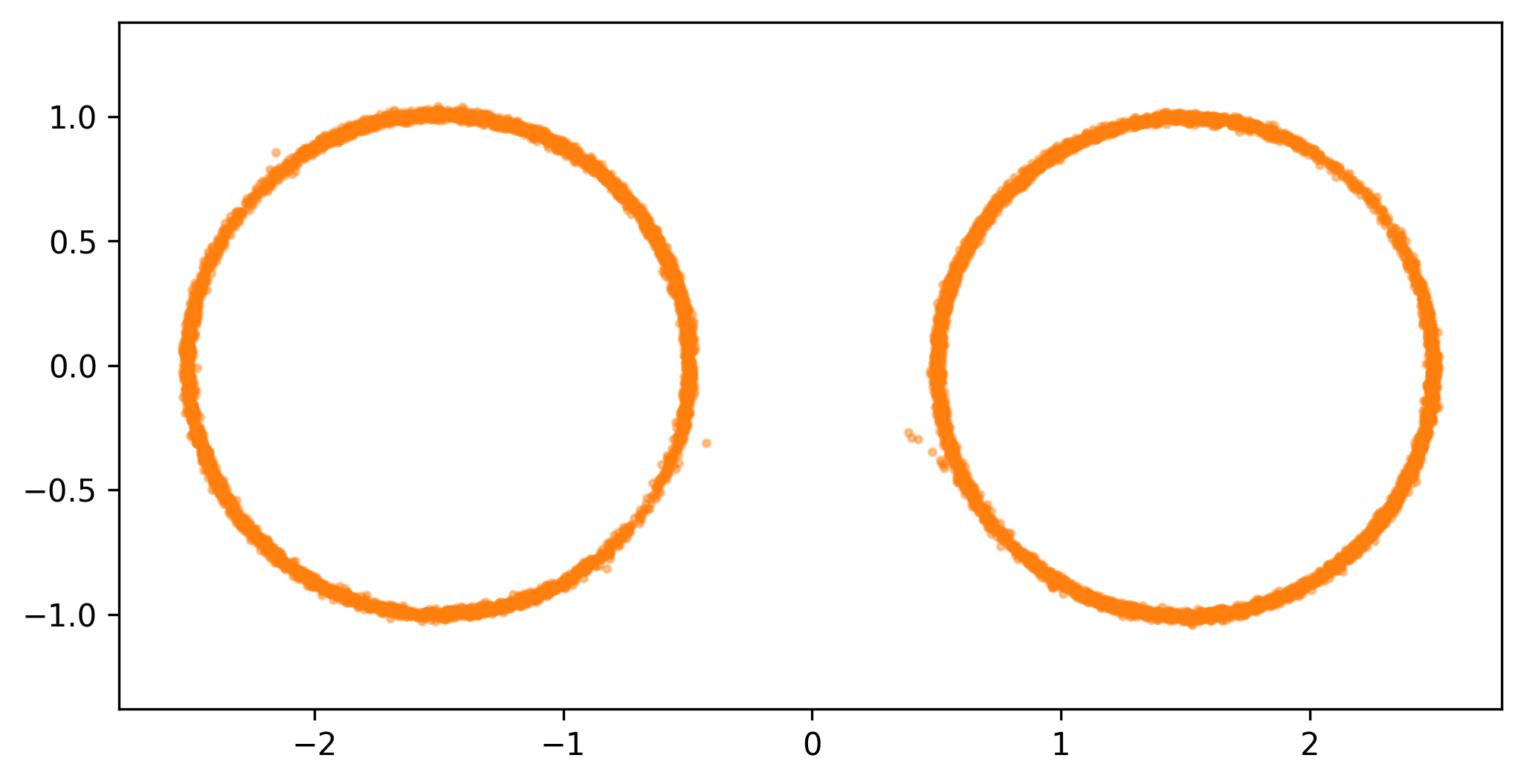}
        \caption{}
    \end{subfigure}

    \caption{Multi-chart formulation improves the quality of generative models. (a) Original data; (b) Diffusion in original space; (c) Latent diffusion with 1 chart; (d) Latent diffusion with 4 charts. }
    \label{fig:02}
\end{figure}

As illustrated in Figure~\ref{fig:02}, the multi-chart formulation greatly increases the flexibility and expressiveness of the model, leading to improved performance in downstream generative tasks. 

See Appendix~\ref{manifold-learning} for more details on the procedure.

\textbf{Diffusion Models}. Diffusion models \citep{ho2020denoising, austin2021structured} have achieved tremendous success in data generation over the past few years. Although models obtained via manifold learning already provide a primitive mechanism for generating new samples, directly sampling from VAEs is often suboptimal \resp{since the simple prior $p(\bm{z})$ specified for the training of VAE usually fails to generalize to regions unseen from the dataset. In practice, it is more effective to further refine the latent structure using a diffusion process \citep{rombach2022high}.}

While latent diffusion models are not a new concept, they are typically restricted to a single-VAE framework and do not naturally accommodate manifold-structured data, as illustrated in Figure~\ref{fig:02}. To enable joint generation from $p(\bm{z}, c)$ despite its intractable form, we introduce a joint diffusion process that simultaneously updates continuous variables (corresponding to $\bm{z}$) and discrete variables (corresponding to $c$); which will be introduced in Section~\ref{joint-diffusion}. The proposed latent joint diffusion process not only significantly improves generation quality for data with underlying manifold structure, but also enhances the proposed imputation procedure.

\section{Imputation under Manifold Hypothesis}\label{design}

We present our procedure in this section. Starting with a dataset $\{\bm{x}_k\}_{k=1}^n \subseteq \mathbb{R}^p$ satisfying the manifold hypothesis, we assume the underlying manifold has dimension $d$, which is known (e.g., estimated via a manifold dimension estimator like TwoNN). We first describe the Bayesian procedure for missing data imputation, followed by the construction of a joint diffusion process in the latent space that further enhances it.

\subsection{Imputation via Sampling-Importance-Resampling}

Our target is to sample from \( p(\bm{x}_{\mathrm{mis}} \mid \bm{x}_{\mathrm{obs}}) \) for imputation. To achieve that, we need to first obtain a sample for $\bm{z}$ and $c$. Since 
\begin{equation}
p(\bm{z}, c \mid \bm{x}) = p(\bm{z} \mid \bm{x}, c)p(c \mid \bm{x}),
\end{equation}
and both $p(\bm{z} \mid \bm{x}, c)$ and $p(c \mid \bm{x})$ are available from the manifold learning procedure, we can sample a chart assignment $c \sim p(c \mid \bm{x})$ for each data point $\bm{x}$, and then draw $\bm{z} \sim p(\bm{z} \mid \bm{x}, c)$. This yields a collection of joint latent samples $\{(\bm{z}_k, c_k)\}_{k=1}^n$.

Introducing the latent variables into the conditional density we have
\begin{align}
    p(\bm{x}_\mathrm{mis} \mid \bm{x}_\mathrm{obs}) 
    &= \sum_c \int p(\bm{x}_\mathrm{mis}, \bm{z} \mid \bm{x}_\mathrm{obs},  c)p(c \mid \bm{x}_\mathrm{obs}) d\bm{z}\nonumber\\
    &= \sum_c \int p(\bm{x}_\mathrm{mis} \mid \bm{x}_\mathrm{obs}, \bm{z}, c)p(\bm{z} \mid c, \bm{x}_\mathrm{obs})p(c \mid \bm{x}_\mathrm{obs})d\bm{z}\nonumber\\
    &\doteq \sum_c \int p(\bm{x}_\mathrm{mis} \mid \bm{z}, c)p(\bm{z} \mid c, \bm{x}_\mathrm{obs})p(c \mid \bm{x}_\mathrm{obs})d\bm{z},
\end{align}
where for the last equality we use the fact that with the normal modeling of $p(\bm{x}\mid \bm{z}, c)$ (Equation~\ref{eqn-conds}), $\bm{x}_\mathrm{obs}$ and $\bm{x}_\mathrm{mis}$ are conditionally independent. Combining $p(\bm{z} \mid c, \bm{x}_\mathrm{obs})$ and $p(c \mid \bm{x}_\mathrm{obs})$ yields $p(\bm{z}, c \mid \bm{x}_\mathrm{obs})$, and we obtain
\begin{align}\label{eqn-sampling}
    p(\bm{x}_\mathrm{mis} \mid \bm{x}_\mathrm{obs}) 
    &= \sum_c \int p(\bm{x}_\mathrm{mis} \mid \bm{z}, c)p(\bm{z}, c \mid \bm{x}_\mathrm{obs})d\bm{z}.
\end{align}

According to Equation~\ref{eqn-sampling}, sampling from \(p(\bm{x}_\mathrm{mis} \mid \bm{x}_\mathrm{obs})\) can be decomposed into sampling from \(p(\bm{z}, c \mid \bm{x}_\mathrm{obs})\) first, then sampling from \(p(\bm{x}_\mathrm{mis} \mid \bm{z}, c)\), which is Gaussian since it marginalize a joint Gaussian. However, \(p(\bm{z}, c \mid \bm{x}_\mathrm{obs})\) is not tractable, hence we seek for the alternative SIR procedure described below to sample from it.

\begin{algorithm}[ht]
\caption{Sampling–Importance–Resampling (SIR)}
\label{alg:sir}

\begin{algorithmic}[1]

\STATE Input: $\bm{x}_\mathrm{obs}$, $p(\bm{x}_\mathrm{obs}\mid \bm{z}, c)$, $\{(\bm{z}_k, c_k)\}_{k=1}^n$
\STATE Compute weights $w_k = p(\bm{x}_\mathrm{obs}\mid \bm{z}_k, c_k)$.
\STATE Normalize weights $\tilde{w}_k = w_k/\sum_k w_k$
\STATE Sample $(\bm{z}^*, c^*)$ from $\{(\bm{z}_k, c_k)\}_{k=1}^n$ according to $\tilde{w}_k$.
\STATE Output: $(\bm{z}^*, c^*) \mathrel{\overset{\cdot}{\sim}} p(\bm{z}, c \mid \bm{x}_\mathrm{obs})$
\end{algorithmic}
\end{algorithm}

The above algorithm provides an asymptotically exact sample from $p(\bm{z}, c \mid \bm{x}_\mathrm{obs})$ as $n \to \infty$; see Appendix~\ref{SIR-proof} for a proof of the following result.

\begin{theorem}\label{thm:01}
Assume all learned models are exact and $\{(\bm{z}_k, c_k)\}_{k=1}^n$ is a random sample of $p(\bm{z}, c)$, then the pair $(\bm{z}^*, c^*)$ produced by Algorithm~\ref{alg:sir} converges weakly (in distribution) to $p(\bm{z}, c \mid \bm{x}_\mathrm{obs})$ as $n \to \infty$.
\end{theorem}

\subsection{Joint Latent Diffusion Process}\label{joint-diffusion}

\resp{According to Theorem~\ref{thm:01}, a larger joint latent sample 
$\{(\bm{z}_k,c_k)\}_{k=1}^n$ can lead to improved imputation quality. 
However, when the joint latent samples are obtained directly from the observed training data, their size is inherently limited by the number of available observations. Consequently, the resulting imputation procedure is constrained by the support of the observed training sample and may fail to adequately explore or represent previously unseen regions of the data space. This limitation motivates the use of a joint diffusion process in the latent space, which enables efficient sampling from $p(\bm{z}, c)$ at arbitrary scale. The proposed model not only improves imputation quality but also provides a principled framework for data generation under the manifold hypothesis, as demonstrated in Figure~\ref{fig:03}.}

\begin{figure}[H]
    \centering
    \begin{subfigure}[b]{0.09\textwidth}
        \centering
        \includegraphics[width=\textwidth]{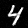}
    \end{subfigure}
    \hfill
    \begin{subfigure}[b]{0.09\textwidth}
        \centering
        \includegraphics[width=\textwidth]{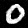}
    \end{subfigure}
    \hfill
    \begin{subfigure}[b]{0.09\textwidth}
        \centering
        \includegraphics[width=\textwidth]{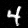}
    \end{subfigure}
    \hfill
    \begin{subfigure}[b]{0.09\textwidth}
        \centering
        \includegraphics[width=\textwidth]{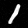}
    \end{subfigure}
    \hfill
    \begin{subfigure}[b]{0.09\textwidth}
        \centering
        \includegraphics[width=\textwidth]{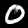}
    \end{subfigure}
    \hfill
    \begin{subfigure}[b]{0.09\textwidth}
        \centering
        \includegraphics[width=\textwidth]{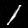}
    \end{subfigure}
    \hfill
    \begin{subfigure}[b]{0.09\textwidth}
        \centering
        \includegraphics[width=\textwidth]{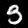}
    \end{subfigure}
    \hfill
    \begin{subfigure}[b]{0.09\textwidth}
        \centering
        \includegraphics[width=\textwidth]{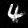}
    \end{subfigure}
    \hfill
    \begin{subfigure}[b]{0.09\textwidth}
        \centering
        \includegraphics[width=\textwidth]{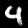}
    \end{subfigure}
    \hfill
    \begin{subfigure}[b]{0.09\textwidth}
        \centering
        \includegraphics[width=\textwidth]{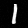}
    \end{subfigure}
    \caption{\resp{Generated MNIST digits using a joint latent diffusion process. The latent space has dimension $d=15$ (compared with the original dimension $p=784$). Four charts are used, each of which naturally corresponds to one digit (only digits 0, 1, 3, and 4 are used for training).}}
    \label{fig:03}
\end{figure}

\resp{A diffusion process is defined by a forward Markov process that gradually adds noise to the data and a reverse Markov process that gradually removes it.} For the forward process, we corrupt $\bm{z}$ and $c$ independently. Specifically,
\begin{equation}
    p(\bm{z}_t, c_t \mid \bm{z}_{t-1}, c_{t-1})
    =
    p(\bm{z}_t \mid \bm{z}_{t-1})
    \, p(c_t \mid c_{t-1}),
\end{equation}
where
\begin{equation}
    \bm{z}_t
    =
    \sqrt{\alpha_t}\,\bm{z}_{t-1}
    +
    \sqrt{1-\alpha_t}\,\bm{\epsilon}_t,
    \qquad
    \bm{\epsilon}_t \sim \mathcal{N}(\bm{0}, I_d),
\end{equation}
and
\begin{equation}
    p(c_t = j \mid c_{t-1} = i)
    =
    Q_t(i,j),
\end{equation}
That is, $\bm{z}_t$ and $c_t$ are corrupted separately through diffusion processes defined on continuous and discrete variables, respectively, with prescribed noise schedules $\{\alpha_t\}_{t=1}^T$ and $\{Q_t\}_{t=1}^T$.

It is important to emphasize that this factorization applies only to the forward corruption process. In particular, we do not assume that $\bm{z}_0$ and $c_0$ are independent under the data distribution, nor does the above construction imply an analogous factorization for the reverse process. 

The one-step factorization of the forward process immediately implies the corresponding $t$-step factorization
\begin{equation}
    p(\bm{z}_t, c_t \mid \bm{z}_0, c_0)
    =
    p(\bm{z}_t \mid \bm{z}_0)\,
    p(c_t \mid c_0),
\end{equation}
which further yields
\begin{equation}
    p(\bm{z}_{t-1}, c_{t-1}
    \mid
    \bm{z}_t, c_t, \bm{z}_0, c_0)
    =
    p(\bm{z}_{t-1} \mid \bm{z}_t, \bm{z}_0)\,
    p(c_{t-1} \mid c_t, c_0).
\end{equation}
Consequently, the posterior
$
p(\bm{z}_{t-1}, c_{t-1}
\mid
\bm{z}_t, c_t, \bm{z}_0, c_0)
$
is fully specified by the forward process and can be computed exactly.

However, the marginal reverse transition
\begin{equation}
    p(\bm{z}_{t-1}, c_{t-1}
    \mid
    \bm{z}_t, c_t)
\end{equation}
does not admit such a decomposition in general. As a result, directly minimizing the Kullback--Leibler divergence between
$
p(\bm{z}_{t-1}, c_{t-1}
\mid
\bm{z}_t, c_t)
$
and
$
p(\bm{z}_{t-1}, c_{t-1}
\mid
\bm{z}_t, c_t, \bm{z}_0, c_0)
$
is intractable.

To address this issue, we follow the standard diffusion-model paradigm and train a neural network to predict the clean latent variables $(\bm{z}_0,c_0)$ directly from $(\bm{z}_t,c_t,t)$. The network therefore serves as an implicit approximation of the reverse dynamics. The resulting training objective consists of a mean-squared-error loss for $\bm{z}_0$ and a cross-entropy loss for $c_0$. Additional implementation details are provided in Appendix~\ref{implementation}.

When sampling, we make use of the following decomposition:
\begin{align}
    p(\bm{z}_{t-1}, c_{t-1} \mid \bm{z}_t, c_t)
    = \sum_{c_0} \int p(\bm{z}_0, c_0 \mid \bm{z}_t, c_t)\,
    p(\bm{z}_{t-1} \mid \bm{z}_t, \bm{z}_0)\,
    p(c_{t-1} \mid c_t, c_0)\, d\bm{z}_0.
\end{align}
This suggests the following one-step reverse procedure: we first draw estimates of $(\bm{z}_0, c_0)$ from the learned model conditioned on $(\bm{z}_t, c_t)$, and then sample $\bm{z}_{t-1}$ and $c_{t-1}$ from the known forward-process posteriors,
\[
p(\bm{z}_{t-1} \mid \bm{z}_t, \bm{z}_0)
\quad \text{and} \quad
p(c_{t-1} \mid c_t, c_0).
\]

\section{Experiments}\label{experiments}

In this section, we first discuss several possible approaches for evaluating imputation quality. We then present experimental results across a wide range of synthetic and real-world datasets. The results show that our proposed method achieves performance comparable to the strongest baseline, MissForest, and outperforms it under high missing rates or complex missing mechanisms.

\subsection{Measurement of Imputation Quality}

The most commonly used method to evaluate imputation quality is root mean squared error (RMSE), which is the sample version of 
\begin{equation}
    \sqrt{\mathbb{E}\left[\|\hat{\bm{x}}_{\mathrm{mis}} - \bm{x}_{\mathrm{mis}}\|_2^2\right]/(p -\|\bm{m}\|_1)},
\end{equation}
in which different variables need to be standardized first before computing. However, this method only focuses on pointwise accuracy and may fail to capture the alignment with the underlying geometry and data distribution. For example, the imputation in Figure~\ref{fig:01b} produced by MissForest has a lower RMSE ($1.5650$) than that in Figure~\ref{fig:01c} produced by our proposed method ($1.8476$), despite the fact that many of its imputed values are clearly misplaced. 

An alternative measure that incorporates both geometry and distributional information is the Wasserstein distance. Let $\mu$ and $\nu$ denote the empirical distributions of the true and imputed data, respectively, the Wasserstein distance is defined as:
\begin{equation}
W_2(\mu, \nu)
= \inf_{\gamma \in \Gamma(\mu, \nu)}
\left(
\mathbb{E}_{(\bm{x}, \bm{y}) \sim \gamma}
\left[ \|\bm{x} - \bm{y}\|_2^2 \right]
\right)^{1/2},
\end{equation}
where $\Gamma(\mu, \nu)$ denotes the set of all \resp{possible joint distributions} between $\mu$ and $\nu$. Ideally, \resp{the norm used} should be defined on the underlying manifold supporting the data distribution; however, this is generally intractable in practice, so we instead use the Euclidean \resp{distance} in $\mathbb{R}^p$. For Figure~\ref{fig:01b} and Figure~\ref{fig:01c}, the Wasserstein distances are $0.3972$ and $0.2753$, respectively.

Sometimes we may wish to directly evaluate model performance on downstream tasks trained using the imputed data, and compare it with models trained on the original complete data. However, the results can vary substantially depending on the choice of model, and in some cases, we may even observe that models trained on imputed data perform better. Data imputation seeks to fill missing entries in a way that is consistent with the main structure of the data; however, there is no guarantee that this structure is relevant to the downstream task.

Besides dataset-specific evaluation metrics, certain methods may also be preferred from a methodological perspective. For instance, inductive methods are often desirable when online imputation is required.  Moreover, since our proposed approach enables direct sampling from $p(\bm{x}_\mathrm{mis} \mid \bm{x}_\mathrm{obs})$, it further allows us to quantify uncertainty rather than merely producing point estimates.

\subsection{Results}

We compare our proposed method with seven competing methods, including mean imputation, PCA imputation, and $k$-nearest neighbor imputation, which are commonly used baselines; MICE and MissForest, which are widely regarded as the strongest classical alternatives; and GAIN and VAE, which are generative-model-based imputation methods similar in spirit to ours.

We first present results on three simulated datasets supported on two circles, a sphere, and a torus, respectively. For the imputation task, each data point has between $1$ and $p-1$ components missing at random. We then evaluate performance on $6$ real-world datasets, all of which have their ambient dimension $p$ larger than the intrinsic dimension $d$ according to manifold dimension estimators, suggesting an underlying manifold structure.

We split each dataset into training and test sets using a $50\%$--$50\%$ ratio, and evaluate different missing rates on the test set. Missing rates are reported as a fraction of the entire dataset size, under the MCAR, MAR, and MNAR mechanisms. For example, a missing rate of $45\%$ corresponds to a $90\%$ missing rate within the test set. To ensure fairness, transductive methods such as MissForest are allowed to operate on the combined dataset, although this provides them with an advantage from an evaluation standpoint.

\begin{table}[htbp]
\centering
\small
\setlength{\tabcolsep}{6pt}
\renewcommand{\arraystretch}{1.25}

\begin{tabular}{lcc|cc|cc}
\toprule

& \multicolumn{2}{c|}{\textbf{Two Circles}}
& \multicolumn{2}{c|}{\textbf{Sphere}}
& \multicolumn{2}{c}{\textbf{Torus}} \\

\cmidrule(lr){2-3} \cmidrule(lr){4-5} \cmidrule(lr){6-7}

\textbf{Method}
& \textbf{RMSE} & \textbf{Wass.}
& \textbf{RMSE} & \textbf{Wass.}
& \textbf{RMSE} & \textbf{Wass.} \\

\midrule

Mean        & 1.2622 & 0.9322 & \textbf{0.5682} & 0.6675 & 1.8143 & 1.1908 \\
PCA         & 1.7673 & 0.8842 & 0.7641 & 0.7612 & 2.4583 & 1.1991 \\
KNN         & 1.2627 & 0.8859 & 0.5722 & 0.6450 & 1.8216 & 1.1451 \\
MICE        & \textbf{1.2395} & 0.6626 & 0.5689 & 0.5460 & 1.7940 & 0.9645 \\
MissForest  & 1.5650 & 0.3972 & 0.6400 & 0.4201 & 1.9585 & 0.7681 \\
GAIN        & 1.2625 & 0.9322 & 0.6152 & 0.6734 & 1.9817 & 1.1940 \\
VAE         & 1.4464 & 0.7850 & 0.8045 & 0.5239 & \textbf{1.2836} & 0.9101 \\
Proposal    & 1.8476 & \textbf{0.2753} & 0.8153 & \textbf{0.2664} & 2.6216 & \textbf{0.4618} \\

\bottomrule
\end{tabular}

\caption{Comparison of imputation methods on synthetic datasets using RMSE and Wasserstein distance. Best results are shown in bold.}
\label{tab:synthetic_results}
\end{table}

We observe from Table~\ref{tab:synthetic_results} that our proposed method achieves the lowest Wasserstein distance across the synthetic datasets, indicating that its imputations best preserve the underlying geometry and data distribution. We also note that, although RMSE is a widely used metric for evaluating imputation quality, it appears to be misleading on these datasets. In particular, methods with lower Wasserstein distances tend to exhibit larger RMSEs. We attribute this to two reasons. First, RMSE evaluates the pointwise error of individual entries and does not account for the geometry or overall distribution of the imputed data. Second, all three datasets are perfectly symmetric around the origin. Consequently, a missing entry may have several equally likely values under the conditional distribution from which our method samples. When the sampled value differs from the particular realization in the ground-truth dataset, RMSE penalizes the imputation heavily, even though the sampled value is consistent with the underlying data-generating mechanism.

Next, we present imputation results on six real-world datasets, evaluated using RMSE and Wasserstein distance. Information on the datasets are provided in Appendix~\ref{implementation}. Detailed results are reported in Tables~\ref{tab:superconductivity:rmse}--\ref{tab:facialexpression:wass} in Appendix~\ref{results}. Visual illustration of the results on the superconductivity dataset are presented in Figure~\ref{fig:05} and Figure~\ref{fig:06}. 

We can see from Figure~\ref{fig:05} that our method achieves the lowest RMSE and Wasserstein distance across all three missing mechanisms among the eight methods, followed by MissForest. A more detailed comparison of the two methods across different missing rates in Figure~\ref{fig:06} reveals that, for the MCAR and MAR cases, MissForest performs better at lower missing rates but is overtaken by our method as the missing rate increases, in terms of both RMSE and Wasserstein distance. In the MNAR case, however, our method outperforms MissForest across almost the entire range of missing rates.

\begin{figure}[H]
    \centering
    \begin{subfigure}[b]{\textwidth}
        \centering
        \includegraphics[width=0.75\textwidth]{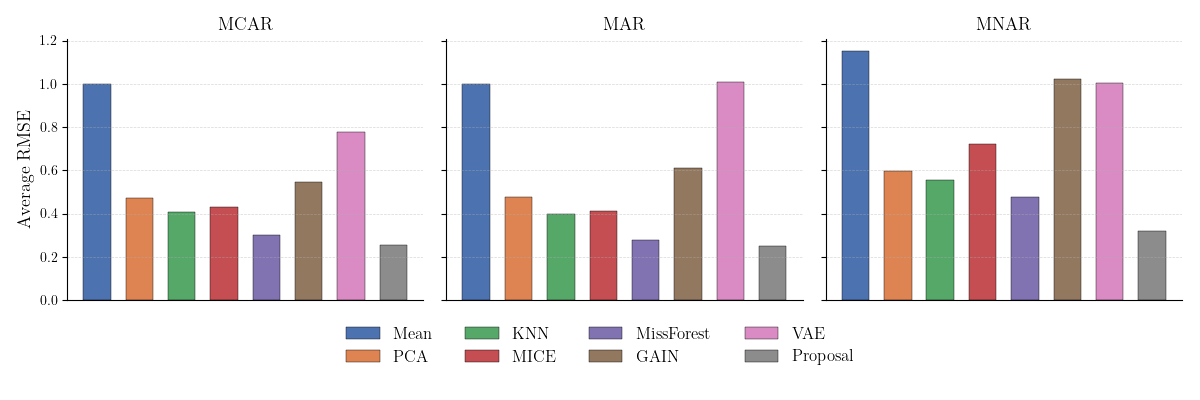}
        \caption{Average RMSE.}
        \label{fig:05a1}
    \end{subfigure}
    \\
    \begin{subfigure}[b]{\textwidth}
        \centering
        \includegraphics[width=0.75\textwidth]{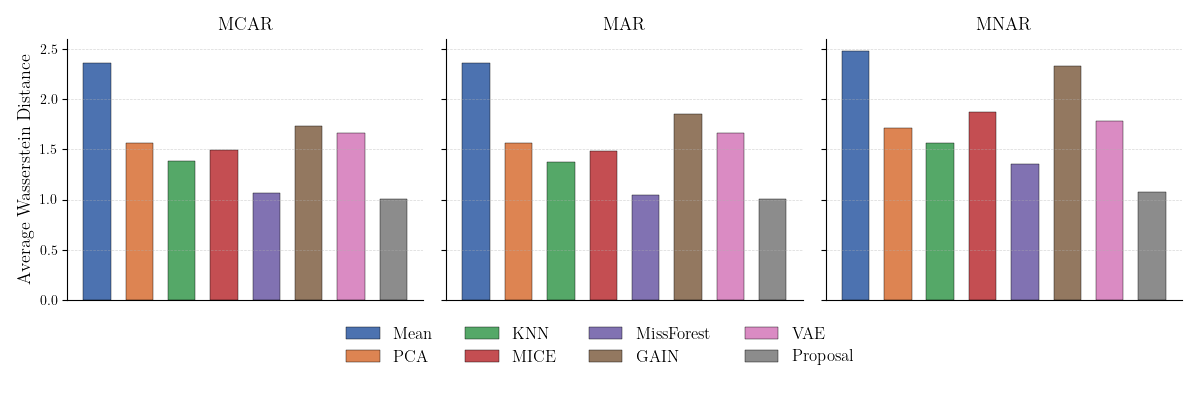}
        \caption{Average Wasserstein distance.}
        \label{fig:05a2}
    \end{subfigure}
    \caption{Average evaluation results over all missing rates, for each method.}
    \label{fig:05}
\end{figure}

\begin{figure}[H]
    \centering
    \begin{subfigure}[b]{\textwidth}
        \centering
        \includegraphics[width=0.75\textwidth]{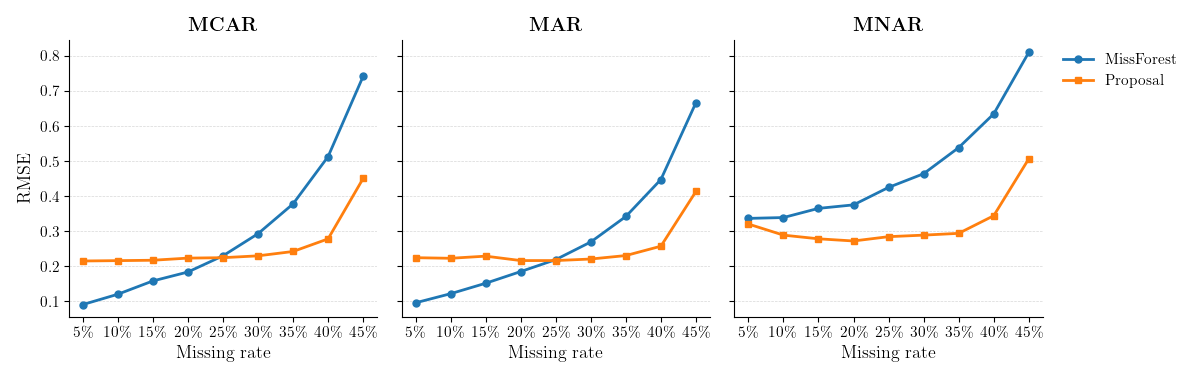}
        \caption{RMSE.}
        \label{fig:05b1}
    \end{subfigure}
    \\
    \begin{subfigure}[b]{\textwidth}
        \centering
        \includegraphics[width=0.75\textwidth]{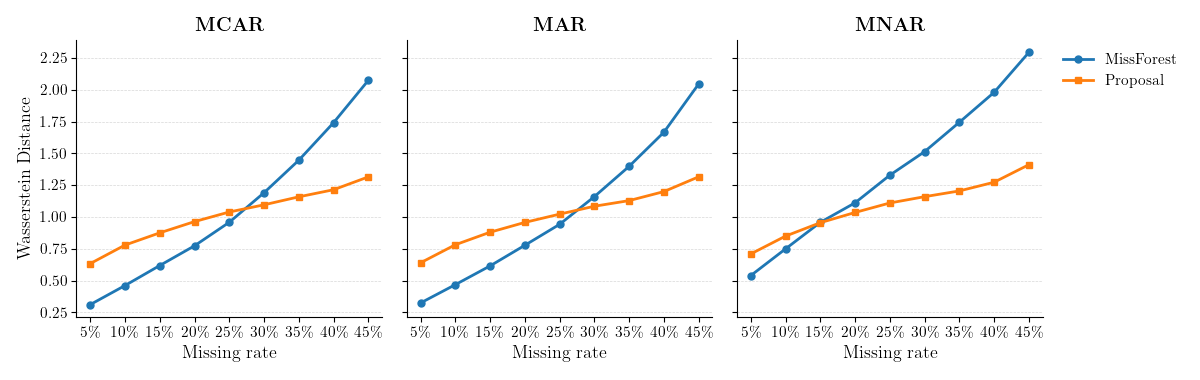}
        \caption{Wasserstein distance.}
        \label{fig:05b2}
    \end{subfigure}
    \caption{Evaluation results across different missing rates, for our method and MissForest.}
    \label{fig:06}
\end{figure}

Similar results can be observed on the energy, finance, and facial expression datasets. For the facial expression dataset, however, the strongest competing method is MICE rather than MissForest, suggesting that the underlying data structure may be closer to a linear setting in this case. On the powerplant dataset, although our method produces relatively large RMSE values, it achieves the smallest Wasserstein distances in most cases, suggesting that this dataset may exhibit a similar phenomenon to that observed in the simulated datasets. In contrast, our method does not perform well on the wine dataset, except for achieving competitive Wasserstein distances under some MNAR settings. This may be due either to the small sample size, which is insufficient to learn a meaningful structure, or to the manifold hypothesis being less appropriate for this dataset. Notably, MissForest also struggles with this dataset, whereas KNN achieves the best overall performance, suggesting that local similarity information may be more effective in this setting.

Given that our method, like many existing imputation methods, is developed under the MCAR assumption, its strong performance under MNAR settings is particularly noteworthy. A possible explanation is that modeling the underlying geometry effectively transfers part of the burden of handling missingness from the missingness mechanism to the data structure itself. Rather than relying solely on assumptions about how entries become missing, our method exploits the relationship among variables (which manifest itself as manifold geometry) to constrain the space of plausible imputations. As a result, the learned structure may remain informative even when the assumed missingness mechanism is violated, reducing the method's sensitivity to the specific mechanism generating the missing entries.

To further evaluate imputation quality from a downstream prediction perspective, we consider the superconductivity dataset, for which we fit a linear regression model using the imputed data, and evaluate its predictive performance using RMSE. The response variable is excluded from the imputation process. Detailed results are reported in Table~\ref{tab:superconductivity:pred}, while the average RMSE across different missing rates under each missing mechanism is presented in Table~\ref{tab:combined_rmse}. Overall, our proposed method achieves the second-best predictive performance under MCAR and MAR, and the best under MNAR.

\begin{table}[H]
\centering
\small
\setlength{\tabcolsep}{5pt}
\renewcommand{\arraystretch}{1.2}

\begin{tabular}{lcccccccc}
\toprule
\textbf{Method} & Mean & PCA & KNN & MICE & MissForest & GAIN & VAE & Proposal \\
\midrule

MCAR & 3.5033 & 1.0740 & 0.3014 & 0.5287 & \textbf{0.2768} & 0.3487 & 0.3399 & 0.2815 \\
MAR  & 3.3649 & 0.6679 & 0.3018 & 0.6091 & \textbf{0.2750} & 0.3986 & 0.3422 & 0.2813 \\
MNAR & 4.1230 & 1.0871 & 0.3041 & 0.7036 & 0.2876 & 0.7967 & 0.4911 & \textbf{0.2807} \\
\bottomrule

\end{tabular}

\caption{Average prediction RMSE under different missing mechanisms on superconductivity dataset. Best results are highlighted in bold.}
\label{tab:combined_rmse}
\end{table}

\section{Conclusion}\label{conclusion}

We have proposed a missing data imputation procedure for data satisfying the manifold hypothesis. The proposed approach is inductive by construction and naturally quantifies imputation uncertainty through probabilistic sampling. Extensive experiments on both synthetic and real-world datasets demonstrate that it is highly competitive with state-of-the-art imputation methods.

Furthermore, by treating $\bm{x}_{\mathrm{mis}}$ as the response variable, the proposed framework can be readily adapted to a nonparametric regression model that accommodates nonlinear relationships in a novel yet natural way. Finally, our current approach samples from $p(\bm{z}, c)$ unconditionally. In future work, $\bm{x}_{\mathrm{obs}}$ could be incorporated as guidance for the diffusion process, allowing it to generate latent samples that are more relevant to the observed data and potentially further improve imputation performance, particularly when the missing rate is low.

\section{Limitation}

Our proposed method relies on the assumption that the data satisfies the manifold hypothesis to a reasonable extent, which is difficult to evaluate in practice. Moreover, since the manifold hypothesis inherently presumes smooth structure, the method is primarily applicable to continuous variables and is not directly suited to discrete data. In addition, learning a reliable imputation model requires a sufficient number of fully observed cases for training. This requirement may be restrictive in settings where complete cases are scarce, even when the overall missing rate is relatively low.

\subsection*{AI use statement}

In this work, we used generative AI tools for English language polishing, grammar checking, and debugging computer code. We have not used generative AI tools for research design, development of research ideas, data analysis, interpretation of results, or generation of scientific content. We have reviewed all AI-assisted work and verified the correctness of any AI-assisted code through testing. We take responsibility for the final content of this work, including text, claims, and artifacts produced with the aid of generative AI.

\subsection*{Reproducibility statement}

To facilitate reproducibility, the datasets and code used to reproduce all experimental results are provided as supplementary materials. They will be made publicly available in a GitHub repository upon completion of the review process.

\subsubsection*{Acknowledgments}
This research includes computations using the computational cluster Katana supported by Research Technology Services at UNSW Sydney.

\bibliography{iclr2027_conference}
\bibliographystyle{iclr2027_conference}

\appendix
\section{Manifold Learning via Mixture VAEs}\label{manifold-learning}

The manifold learning framework by \cite{alberti2024manifold} is built upon the ideas of mixtures of variational autoencoders \citep{ye2020mixtures} and normalizing flows \citep{rezende2015variational}. Its structure is represented by $E_c$ and $D_c$ for different charts defined in Equation~\ref{eqn-conds}. In particular, we define
\[
D_c = \tilde{D}_c \circ T, \qquad E_c = T^{-1} \circ \tilde{E}_c,
\]
for a normalizing flow $T:\mathbb{R}^d \to \mathbb{R}^d$. We adopt the implementation of \cite{alberti2024manifold}, where the network architecture is described in detail. In the following, we briefly introduce the training procedure and the loss function. Let $\xi \sim \mathcal{N}(\bm{0}, I_d)$. Using the reparameterization trick, the evidence lower bound (ELBO; \cite{kingma2013auto}) for a single chart can be written as
\begin{align}
\mathrm{ELBO}_c(\bm{x}) = \mathbb{E}_{\xi}\Big[ & \log p_{\bm{z}}\big(T^{-1}(\tilde{E}_c(\bm{x}) + \sigma_z \xi)\big) + \log \big|\det \nabla T^{-1}(\tilde{E}_c(\bm{x}) + \sigma_z \xi)\big| \nonumber \\
& - \frac{1}{2\sigma_x^2} \big\| \tilde{D}_c(\tilde{E}_c(\bm{x}) + \sigma_z \xi) - \bm{x}\big\|^2 \Big],
\end{align}
where $p_{\bm{z}}$ is the density function of the simple prior specified for $\bm{z}$.

\resp{Let $\beta_c=p(c\mid\bm{x})$ and $\alpha_c=p(c)$. For each training step, $\beta_c$ is first estimated by
approximating $\log p(\bm{x}\mid c)$ with the chart-specific ELBO,
$\mathrm{ELBO}_c(\bm{x})$. Specifically,
\begin{equation}
\beta_c
=
p(c\mid\bm{x})
=
\frac{p(c)p(\bm{x}\mid c)}
{\sum_{c'}p(c')p(\bm{x}\mid c')}
\doteq
\frac{
\alpha_c\exp\left[\mathrm{ELBO}_c(\bm{x})\right]
}{
\sum_{c'}\alpha_{c'}\exp\left[\mathrm{ELBO}_{c'}(\bm{x})\right]
}.
\end{equation}
The chart probabilities are then updated using the estimated posterior
probabilities:
\begin{equation}
\alpha_c
=
\frac{1}{n}\sum_{k=1}^{n}p(c\mid\bm{x}_k)
=
\frac{1}{n}\sum_{k=1}^{n}\hat{\beta}_{kc},
\end{equation}
where $\hat{\beta}_{kc}$ denotes the estimated probability that
$\bm{x}_k$ belongs to chart $c$.
\\
Given the posterior probabilities, the marginal log-likelihood can be
lower bounded using Jensen's inequality:
\begin{align}
\log p(\bm{x})
&=
\log\sum_c\beta_c
\frac{\alpha_c p(\bm{x}\mid c)}{\beta_c}
\geq
\sum_c\beta_c\log p(\bm{x}\mid c)
+
\sum_c\beta_c\log\frac{\alpha_c}{\beta_c}.
\end{align}
During the optimization of the encoder and decoder parameters,
$\alpha_c$ and $\beta_c$ are treated as fixed quantities, as they are
updated separately.
Consequently, terms involving only $\alpha_c$ and $\beta_c$ do not
contribute to the gradient update hence can be removed. Replacing $\log p(\bm{x}\mid c)$ with
the corresponding chart-specific ELBO and using the estimated posterior
probabilities $\hat{\beta}_c$, the objective for gradient-based
optimization is therefore
\begin{equation}\label{loss-01}
\mathcal{L}(\theta)
=
-\sum_c\hat{\beta}_c
\mathrm{ELBO}_c(\bm{x}).
\end{equation}
}

Two further refinements are made during the training process. First, because the loss term in Equation~\ref{loss-01} is purely probabilistic and ignores the underlying manifold geometry, a geometric penalty term is introduced during the initial training phase. This penalty forces data points that are close in the original high-dimensional space to remain close in the latent space, thereby ensuring that neighboring points are more likely to be assigned to the same chart. Second, since the charts covering a manifold are open sets that can overlap by definition, data points that are close to the boundary of one chart can possibly belong to multiple charts. To encourage that, a modified estimation of $\beta_c$ is introduced which leads to a modified loss function that is applied during the final stage of training.

\section{Theoretical Justification of SIR}\label{SIR-proof}

\addtocounter{theorem}{-1}
\begin{theorem}
Assume all learned models are exact and $\{(\bm{z}_k, c_k)\}_{k=1}^n$ is a random sample of $p(\bm{z}, c)$, then the pair $(\bm{z}^*, c^*)$ produced by Algorithm~\ref{alg:sir} converges weakly (in distribution) to $p(\bm{z}, c \mid \bm{x}_\mathrm{obs})$ as $n \to \infty$.
\end{theorem}

\begin{proof}
    Suppose $\{\bm{z}_k, c_k\}_{k=1}^n$ is an independent, identically distributed sample from $p(\bm{z}, c)$. According to Algorithm~\ref{alg:sir}, the normalized weights we use for resampling are defined as
    \begin{equation}
        \tilde{w}_k
        =
        \frac{
        p(\bm x_{\mathrm{obs}}\mid \bm{z}_k, c_k)
        }{
        \sum_{k=1}^n
        p(\bm x_{\mathrm{obs}}\mid \bm{z}_k, c_k)
        },
    \end{equation}
    which gives a discrete distribution for resampling, denoted as $p_{\tilde{w}}(\bm{z}, c)$.

    We want to show $(\bm{z}^*, c^*)$ sampled according to $\tilde{w}_k$ converges weakly to $p(\bm{z}, c \mid \bm x_{\mathrm{obs}})$, which is equivalent to show that the discrete distribution $p_{\tilde{w}}(\bm{z}, c)$ it is sampled from converges weakly to the target distribution as $n \to \infty$.
    
    Consider a bounded continuous function $f(\bm{z}, c)$, whose mean under $p_{\tilde{w}}(\bm{z}, c)$ is
    \begin{equation}
    \mathbb{E}_{p_{\tilde{w}}(\bm{z}, c)}\big[f(\bm{z}, c)\big]
    =
    \sum_{k} \tilde{w}_k f(\bm{z}_k, c_k)
    =
    \frac{
    1/n \sum_{k}
    f(\bm{z}_k, c_k)\, p(\bm{x}_{\mathrm{obs}} \mid \bm{z}_k, c_k)
    }{
    1/n \sum_{k}
    p(\bm{x}_{\mathrm{obs}} \mid \bm{z}_k, c_k)
    }.
    \end{equation}
    Since $p(\bm{x}_{\mathrm{obs}} \mid \bm{z}, c)$ is Gaussian hence bounded, according to law of large numbers, we have
    \begin{equation}
    \frac{1}{n} \sum_{k} f(\bm{z}_k, c_k)\, p(\bm{x}_{\mathrm{obs}} \mid \bm{z}_k, c_k) \xrightarrow{\mathrm{a.s.}} \mathbb{E}_{p(\bm{z}, c)}[f(\bm{z}, c)p(\bm{x}_{\mathrm{obs}} \mid \bm{z}, c)] 
    \end{equation}
    and
    \begin{equation}
        \frac{1}{n}\sum_{k}
    p(\bm{x}_{\mathrm{obs}} \mid \bm{z}_k, c_k) \xrightarrow{\mathrm{a.s.}} \mathbb{E}_{p(\bm z,c)}
    p(\bm x_{\mathrm{obs}}\mid \bm{z}, c)
    =
    p(\bm x_{\mathrm{obs}}).
    \end{equation}

    Finally, notice that
    \begin{equation}
        \frac{\mathbb{E}_{p(\bm{z}, c)}[f(\bm{z}, c)p(\bm{x}_{\mathrm{obs}} \mid \bm{z}, c)]}{p(\bm x_{\mathrm{obs}})} = \int \sum_cf(\bm{z}, c)p(\bm{z}, c \mid \bm{x}_\mathrm{obs})d\bm{z} = \mathbb{E}_{p(\bm{z}, c \mid \bm{x}_\mathrm{obs})}[f(\bm{z}, c)],
    \end{equation}
    we know $\mathbb{E}_{p_{\tilde{w}}(\bm{z}, c)}\big[f(\bm{z}, c)\big]$ converges almost surely to $\mathbb{E}_{p(\bm{z}, c \mid \bm{x}_\mathrm{obs})}[f(\bm{z}, c)]$ as $n \to \infty$ and the proof is complete.
\end{proof}

\section{Implementation Details}\label{implementation}

\subsection{Real-world Datasets}

The six real-world datasets are either obtained from the UCI Machine Learning Repository or collected by the authors, which are available in the supplementary materials accompanying the manuscript.

The superconductivity dataset contains records on 21263 superconductors along with their associated material features. A total of 78 continuous features are used for the imputation tasks, while the critical temperature is treated as the response variable for downstream evaluation. The estimated intrinsic dimension of this dataset is $d = 2$. We choose the number of charts to be $6$.

The energy dataset contains 19735 time-stamped observations of indoor temperature and humidity, energy consumption, and external weather variables aligned via timestamps. A total of 24 continuous features are used for the imputation tasks. The estimated intrinsic dimension of this dataset is $d = 3$. We choose the number of charts to be $6$.

The finance dataset consists of 45959 records of 15-minute interval closing prices for 50 highly liquid U.S. stocks. The estimated intrinsic dimension of this dataset is $d = 4$. We choose the number of charts to be $6$.

The powerplant dataset contains 9568 records collected from a combined cycle power plant. A total of 4 continuous features are used for the imputation tasks. The estimated intrinsic dimension of this dataset is $d = 3$. We choose the number of charts to be $2$.

The wine dataset contains 4898 records on quantities of different constituents found in wine of different qualities. A total of 11 continuous features are used for the imputation tasks. The estimated intrinsic dimension of this dataset is $d = 2$. We choose the number of charts to be $1$.

The facial expression dataset contains 27934 records of landmark coordinates extracted from facial expression frames. A total of 200 continuous features are used for the imputation tasks. The estimated intrinsic dimension of this dataset is $d = 3$. We choose the number of charts to be $6$.

\subsection{Imputation Settings}

Since the joint diffusion process is carried out in a latent space of much lower dimension, thanks to the manifold hypothesis, relatively simple neural networks can already achieve strong performance. Our model uses a neural network with a hidden dimension of 256, which serves as a shared representation space for all inputs. The continuous variable $\bm{z}_t$ is concatenated with a time embedding for $t$ and a learned embedding of the discrete class label for $c_t$; both are mapped into the same hidden dimension through small MLPs. This combined representation is then passed through a shared trunk consisting of two fully connected layers with SiLU activations, which allows interaction between the continuous and discrete components. Finally, two linear output heads are used: one maps the hidden representation back to the continuous data dimension to predict $\bm{z}_0$, and the other produces logits for predicting $c_0$. This simple shared-hidden-dimension design enables efficient learning of joint dynamics while keeping the architecture lightweight and scalable.

The total number of diffusion steps for both the forward and reverse processes is set to $T = 500$. The corruption schedule is parameterized through $\alpha_t$, which is linearly decreasing from $1$, and is applied to both continuous and discrete variables. In particular, if multiple charts are used, we have
\begin{equation}
\alpha_t = 1 - 10^{-4} t, 
\qquad
Q_t(i,i) = \alpha_t, 
\qquad
Q_t(i,j) = \frac{1 - \alpha_t}{C - 1}, \; i \neq j.
\end{equation}
The size of the joint latent sample is chosen to be $10000$ for all datasets.

More powerful neural network architectures, as well as more careful tuning of the diffusion schedule and related hyperparameters, may further improve the performance of the proposed method, and are left for future exploration.

\section{Experimental Results}\label{results}

\begin{table}[htbp]
\centering
\small
\setlength{\tabcolsep}{4pt}

\begin{tabular}{c|c|cccccccc}
\toprule
\multicolumn{10}{c}{\textbf{RMSE}} \\
\midrule
\textbf{Mechanism} & \textbf{Rate} & Mean & PCA & KNN & MICE & MissForest & GAIN & VAE & Proposal \\
\midrule
\multirow{9}{*}{MCAR}
& 5\%  & 0.9953 & 0.3075 & 0.3722 & 0.1454 & \textbf{0.0908} & 0.3759 & 0.4425 & 0.2153 \\
& 10\% & 0.9960 & 0.3254 & 0.3702 & 0.1799 & \textbf{0.1203} & 0.3854 & 0.4906 & 0.2163 \\
& 15\% & 0.9956 & 0.3528 & 0.3750 & 0.2254 & \textbf{0.1583} & 0.4122 & 1.0150 & 0.2174 \\
& 20\% & 0.9976 & 0.3812 & 0.3816 & 0.2849 & \textbf{0.1838} & 0.4539 & 0.6065 & 0.2235 \\
& 25\% & 0.9972 & 0.4200 & 0.3864 & 0.3596 & 0.2300 & 0.4750 & 0.6611 & \textbf{0.2247} \\
& 30\% & 0.9971 & 0.4804 & 0.3954 & 0.4563 & 0.2931 & 0.5480 & 0.7357 & \textbf{0.2300} \\
& 35\% & 0.9968 & 0.5556 & 0.4092 & 0.5747 & 0.3782 & 0.6309 & 0.8523 & \textbf{0.2424} \\
& 40\% & 0.9970 & 0.6377 & 0.4385 & 0.7320 & 0.5133 & 0.7457 & 0.9496 & \textbf{0.2784} \\
& 45\% & 0.9972 & 0.7953 & 0.5170 & 0.9258 & 0.7425 & 0.8900 & 1.2201 & \textbf{0.4509} \\
\midrule

\multirow{9}{*}{MAR}
& 5\%  & 1.0538 & 0.3530 & 0.4015 & 0.1553 & \textbf{0.0959} & 0.3681 & 0.5107 & 0.2247 \\
& 10\% & 1.0227 & 0.3551 & 0.3848 & 0.1854 & \textbf{0.1219} & 0.3740 & 0.5066 & 0.2230 \\
& 15\% & 1.0049 & 0.3862 & 0.3777 & 0.2247 & \textbf{0.1517} & 0.4009 & 0.5499 & 0.2291 \\
& 20\% & 0.9926 & 0.4003 & 0.3745 & 0.2760 & \textbf{0.1851} & 0.4270 & 1.3008 & 0.2165 \\
& 25\% & 0.9852 & 0.4347 & 0.3758 & 0.3421 & 0.2192 & 0.5147 & 0.6379 & \textbf{0.2165} \\
& 30\% & 0.9804 & 0.4768 & 0.3801 & 0.4244 & 0.2696 & 0.6332 & 0.7024 & \textbf{0.2209} \\
& 35\% & 0.9785 & 0.5346 & 0.3888 & 0.5376 & 0.3423 & 0.7915 & 0.7948 & \textbf{0.2308} \\
& 40\% & 0.9802 & 0.6067 & 0.4134 & 0.6940 & 0.4479 & 0.9266 & 0.9294 & \textbf{0.2575} \\
& 45\% & 0.9841 & 0.7509 & 0.4841 & 0.8776 & 0.6673 & 1.0721 & 3.1421 & \textbf{0.4138} \\
\midrule

\multirow{9}{*}{MNAR}
& 5\%  & 1.5822 & 0.5610 & 0.7158 & 0.6185 & 0.3369 & 0.6472 & 0.8736 & \textbf{0.3208} \\
& 10\% & 1.3428 & 0.5160 & 0.6186 & 0.6277 & 0.3393 & 0.7462 & 0.8355 & \textbf{0.2891} \\
& 15\% & 1.2152 & 0.5169 & 0.5694 & 0.6519 & 0.3655 & 0.7663 & 1.1314 & \textbf{0.2786} \\
& 20\% & 1.1309 & 0.5192 & 0.5347 & 0.6656 & 0.3755 & 0.8808 & 0.8403 & \textbf{0.2724} \\
& 25\% & 1.0729 & 0.5441 & 0.5153 & 0.6897 & 0.4254 & 1.0359 & 1.2839 & \textbf{0.2847} \\
& 30\% & 1.0316 & 0.5741 & 0.5042 & 0.7212 & 0.4644 & 1.0546 & 0.8849 & \textbf{0.2891} \\
& 35\% & 1.0042 & 0.6291 & 0.4967 & 0.7681 & 0.5389 & 1.1305 & - & \textbf{0.2942} \\
& 40\% & 0.9887 & 0.7061 & 0.5028 & 0.8364 & 0.6352 & 1.3300 & - & \textbf{0.3445} \\
& 45\% & 0.9846 & 0.8080 & 0.5479 & 0.9235 & 0.8108 & 1.5958 & 1.1745 & \textbf{0.5069}\\
\bottomrule
\end{tabular}
\caption{Comparison of imputation methods on the superconductivity dataset under different missing mechanisms and missing rates using RMSE. Best results are shown in bold, - means the method either collapses or the value is too big to be meaningful.}
\label{tab:superconductivity:rmse}
\end{table}

\begin{table}[htbp]
\centering
\small
\setlength{\tabcolsep}{4pt}

\begin{tabular}{c|c|cccccccc}
\toprule
\multicolumn{10}{c}{\textbf{Wasserstein Distance}} \\
\midrule
\textbf{Mechanism} & \textbf{Rate} & Mean & PCA & KNN & MICE & MissForest & GAIN & VAE & Proposal \\
\midrule

\multirow{9}{*}{MCAR}
& 5\%  & 1.5916 & 0.8434 & 0.8676 & 0.5518 & \textbf{0.3109} & 0.9450 & 0.9019 & 0.6334 \\
& 10\% & 1.9330 & 1.0494 & 1.0594 & 0.7536 & \textbf{0.4608} & 1.1642 & 1.1914 & 0.7791 \\
& 15\% & 2.1588 & 1.2214 & 1.2017 & 0.9526 & \textbf{0.6182} & 1.3477 & 1.4369 & 0.8753 \\
& 20\% & 2.3288 & 1.3683 & 1.3246 & 1.1687 & \textbf{0.7734} & 1.5451 & 1.6147 & 0.9637 \\
& 25\% & 2.4607 & 1.5217 & 1.4196 & 1.4111 & \textbf{0.9584} & 1.6770 & 1.7704 & 1.0393 \\
& 30\% & 2.5648 & 1.7103 & 1.5176 & 1.6878 & 1.1909 & 1.8941 & 1.8948 & \textbf{1.0958} \\
& 35\% & 2.6529 & 1.8918 & 1.5968 & 1.9812 & 1.4466 & 2.1001 & 2.0003 & \textbf{1.1586} \\
& 40\% & 2.7251 & 2.0672 & 1.6748 & 2.3016 & 1.7418 & 2.3462 & 2.0621 & \textbf{1.2145} \\
& 45\% & 2.7886 & 2.3753 & 1.7731 & 2.6412 & 2.0749 & 2.6054 & 2.1155 & \textbf{1.3154} \\
\midrule

\multirow{9}{*}{MAR}
& 5\%  & 1.6163 & 0.8645 & 0.8830 & 0.5600 & \textbf{0.3235} & 0.9287 & 0.9257 & 0.6395 \\
& 10\% & 1.9470 & 1.0622 & 1.0694 & 0.7555 & \textbf{0.4675} & 1.1338 & 1.1901 & 0.7814 \\
& 15\% & 2.1640 & 1.2416 & 1.2081 & 0.9453 & \textbf{0.6150} & 1.3317 & 1.4211 & 0.8801 \\
& 20\% & 2.3257 & 1.3802 & 1.3122 & 1.1524 & \textbf{0.7772} & 1.5008 & 1.6218 & 0.9573 \\
& 25\% & 2.4537 & 1.5448 & 1.4087 & 1.3890 & \textbf{0.9420} & 1.7661 & 1.7472 & 1.0223 \\
& 30\% & 2.5590 & 1.7062 & 1.4923 & 1.6488 & 1.1598 & 2.0518 & 1.8672 & \textbf{1.0846} \\
& 35\% & 2.6451 & 1.8748 & 1.5694 & 1.9466 & 1.3980 & 2.3927 & 1.9707 & \textbf{1.1288} \\
& 40\% & 2.7197 & 2.0496 & 1.6501 & 2.2882 & 1.6662 & 2.6445 & 2.0592 & \textbf{1.1997} \\
& 45\% & 2.7836 & 2.3414 & 1.7552 & 2.6162 & 2.0455 & 2.8698 & 2.1775 & \textbf{1.3159} \\
\midrule

\multirow{9}{*}{MNAR}
& 5\%  & 1.9587 & 1.0195 & 1.1355 & 0.9237 & \textbf{0.5391} & 1.1500 & 1.1777 & 0.7094 \\
& 10\% & 2.2048 & 1.2389 & 1.3271 & 1.2532 & \textbf{0.7493} & 1.5504 & 1.4801 & 0.8497 \\
& 15\% & 2.3503 & 1.4238 & 1.4572 & 1.5232 & 0.9579 & 1.7869 & 1.7090 & \textbf{0.9551} \\
& 20\% & 2.4507 & 1.5615 & 1.5384 & 1.7295 & 1.1115 & 2.0598 & 1.8373 & \textbf{1.0357} \\
& 25\% & 2.5323 & 1.7092 & 1.6139 & 1.9158 & 1.3288 & 2.4277 & 1.9723 & \textbf{1.1102} \\
& 30\% & 2.6013 & 1.8407 & 1.6718 & 2.0901 & 1.5140 & 2.5183 & 2.0415 & \textbf{1.1598} \\
& 35\% & 2.6645 & 1.9987 & 1.7223 & 2.2706 & 1.7450 & 2.7547 & - & \textbf{1.2050} \\
& 40\% & 2.7272 & 2.1899 & 1.7739 & 2.4662 & 1.9805 & 3.1448 & - & \textbf{1.2734} \\
& 45\% & 2.7879 & 2.4342 & 1.8508 & 2.6672 & 2.2948 & 3.5493 & 2.2400 & \textbf{1.4113}\\
\bottomrule
\end{tabular}

\caption{Comparison of imputation methods on the superconductivity dataset under different missing mechanisms and missing rates using Wasserstein distance. Best results are shown in bold, - means the method either collapses or the value is too big to be meaningful.}
\label{tab:superconductivity:wass}
\end{table}

\begin{table}[htbp]
\centering
\small
\setlength{\tabcolsep}{4pt}

\begin{tabular}{c|c|cccccccc}
\toprule
\multicolumn{10}{c}{\textbf{RMSE}} \\
\midrule
\textbf{Mechanism} & \textbf{Rate} & Mean & PCA & KNN & MICE & MissForest & GAIN & VAE & Proposal \\
\midrule

\multirow{9}{*}{MCAR}
& 5\%  & 0.9972 & 0.9305 & 0.4324 & 0.4488 & \textbf{0.1320} & 0.6380 & 0.4559 & 0.2640 \\
& 10\% & 0.9992 & 0.9254 & 0.4343 & 0.4595 & \textbf{0.1610} & 0.6223 & 0.5110 & 0.2735 \\
& 15\% & 0.9991 & 0.9179 & 0.4379 & 0.4751 & \textbf{0.1978} & 0.6528 & - & 0.2703 \\
& 20\% & 0.9980 & 0.8978 & 0.4431 & 0.4958 & \textbf{0.2529} & 0.6400 & - & 0.2870 \\
& 25\% & 0.9984 & 0.8762 & 0.4543 & 0.5261 & 0.3212 & 0.6613 & - & \textbf{0.2982} \\
& 30\% & 0.9971 & 0.8509 & 0.4712 & 0.5648 & 0.4171 & 0.6619 & - & \textbf{0.3497} \\
& 35\% & 0.9970 & 0.8325 & 0.5058 & 0.6252 & 0.5309 & 0.6870 & - & \textbf{0.4557} \\
& 40\% & 0.9972 & 0.8603 & 0.5721 & 0.7149 & 0.6833 & 0.7960 & - & \textbf{0.6476} \\
& 45\% & 0.9967 & 0.9245 & \textbf{0.7229} & 0.8440 & 0.8698 & 0.9029 & - & 0.9880 \\
\midrule

\multirow{9}{*}{MAR}
& 5\%  & 1.0718 & 0.9137 & 0.4172 & 0.4342 & \textbf{0.1217} & 0.6065 & 0.4632 & 0.2502 \\
& 10\% & 1.0500 & 0.9117 & 0.4248 & 0.4546 & \textbf{0.1565} & 0.5800 & 0.5214 & 0.2552 \\
& 15\% & 1.0372 & 0.9160 & 0.4324 & 0.4768 & \textbf{0.1964} & 0.6344 & - & 0.2690 \\
& 20\% & 1.0243 & 0.8849 & 0.4389 & 0.4969 & \textbf{0.2497} & 0.6098 & - & 0.2780 \\
& 25\% & 1.0137 & 0.8602 & 0.4503 & 0.5267 & 0.3125 & 0.7369 & - & \textbf{0.2894} \\
& 30\% & 1.0043 & 0.8256 & 0.4670 & 0.5649 & 0.3990 & 0.7590 & - & \textbf{0.3381} \\
& 35\% & 0.9989 & 0.8408 & 0.4962 & 0.6214 & 0.5014 & 0.8280 & - & \textbf{0.4294} \\
& 40\% & 0.9961 & 0.8733 & \textbf{0.5668} & 0.7107 & 0.6442 & 0.9290 & - & 0.6443 \\
& 45\% & 0.9947 & 0.9221 & \textbf{0.7101} & 0.8320 & 0.8309 & 0.9893 & - & 0.9700 \\
\midrule

\multirow{9}{*}{MNAR}
& 5\%  & 1.3612 & 0.9810 & 0.6027 & 0.6362 & \textbf{0.2313} & 0.8039 & 0.6495 & 0.3368 \\
& 10\% & 1.2339 & 0.9400 & 0.5461 & 0.6112 & \textbf{0.2484} & 0.6675 & 0.8799 & 0.3108 \\
& 15\% & 1.1502 & 0.9382 & 0.5190 & 0.6198 & \textbf{0.2769} & 0.7399 & - & 0.3064 \\
& 20\% & 1.0891 & 0.9313 & 0.5117 & 0.6495 & 0.3316 & 0.7905 & - & \textbf{0.3165} \\
& 25\% & 1.0453 & 0.9221 & 0.5200 & 0.6864 & 0.4182 & 1.1245 & - & \textbf{0.3422} \\
& 30\% & 1.0128 & 0.9431 & 0.5460 & 0.7265 & 0.5012 & 1.0835 & - & \textbf{0.4210} \\
& 35\% & 0.9920 & 0.9191 & 0.5884 & 0.7727 & 0.6104 & 1.1634 & - & \textbf{0.5590} \\
& 40\% & 0.9816 & 0.9007 & \textbf{0.6577} & 0.8272 & 0.7349 & 1.1393 & - & 0.7511 \\
& 45\% & 0.9821 & 0.9150 & \textbf{0.7621} & 0.8878 & 0.9190 & 1.1339 & - & 1.0172\\
\bottomrule
\end{tabular}

\caption{Comparison of imputation methods on the energy dataset under different missing mechanisms and missing rates using RMSE. Best results are shown in bold, - means the method either collapses or the value is too big to be meaningful.}
\label{tab:energy:rmse}
\end{table}

\begin{table}[htbp]
\centering
\small
\setlength{\tabcolsep}{4pt}

\begin{tabular}{c|c|cccccccc}
\toprule
\multicolumn{10}{c}{\textbf{Wasserstein Distance}} \\
\midrule
\textbf{Mechanism} & \textbf{Rate} & Mean & PCA & KNN & MICE & MissForest & GAIN & VAE & Proposal \\
\midrule

\multirow{9}{*}{MCAR}
& 5\%  & 1.1397 & 0.9517 & 0.7051 & 0.7072 & \textbf{0.3024} & 0.8845 & 0.5696 & 0.4861 \\
& 10\% & 1.4171 & 1.2163 & 0.9045 & 0.9216 & \textbf{0.4383} & 1.1005 & 0.7835 & 0.6308 \\
& 15\% & 1.5847 & 1.3744 & 1.0315 & 1.0689 & \textbf{0.5748} & 1.2710 & - & 0.7294 \\
& 20\% & 1.7020 & 1.4670 & 1.1243 & 1.1820 & \textbf{0.7314} & 1.3469 & - & 0.7954 \\
& 25\% & 1.7921 & 1.5260 & 1.2045 & 1.2850 & 0.8797 & 1.4438 & - & \textbf{0.8574} \\
& 30\% & 1.8660 & 1.5577 & 1.2649 & 1.3756 & 1.0437 & 1.4931 & - & \textbf{0.9143} \\
& 35\% & 1.9292 & 1.6073 & 1.3267 & 1.4698 & 1.1929 & 1.5534 & - & \textbf{0.9846} \\
& 40\% & 1.9897 & 1.7021 & 1.3983 & 1.5891 & 1.3344 & 1.6992 & - & \textbf{1.0565} \\
& 45\% & 2.0515 & 1.8900 & 1.5516 & 1.7716 & 1.5124 & 1.8672 & - & \textbf{1.1506} \\
\midrule

\multirow{9}{*}{MAR}
& 5\%  & 1.1717 & 0.9421 & 0.6949 & 0.6952 & \textbf{0.3002} & 0.8472 & 0.5710 & 0.4842 \\
& 10\% & 1.4461 & 1.2153 & 0.8964 & 0.9175 & \textbf{0.4365} & 1.0578 & 0.7758 & 0.6218 \\
& 15\% & 1.6122 & 1.3769 & 1.0286 & 1.0703 & \textbf{0.5725} & 1.2536 & - & 0.7232 \\
& 20\% & 1.7221 & 1.4647 & 1.1215 & 1.1871 & \textbf{0.7258} & 1.3275 & - & 0.7951 \\
& 25\% & 1.8078 & 1.5211 & 1.2021 & 1.2929 & 0.8806 & 1.5353 & - & \textbf{0.8556} \\
& 30\% & 1.8777 & 1.5471 & 1.2665 & 1.3856 & 1.0369 & 1.6239 & - & \textbf{0.9197} \\
& 35\% & 1.9385 & 1.6273 & 1.3216 & 1.4831 & 1.1825 & 1.7181 & - & \textbf{0.9809} \\
& 40\% & 1.9945 & 1.7170 & 1.3969 & 1.6025 & 1.3303 & 1.8857 & - & \textbf{1.0642} \\
& 45\% & 2.0542 & 1.8949 & 1.5531 & 1.7810 & 1.5148 & 1.9801 & - & \textbf{1.1665} \\
\midrule

\multirow{9}{*}{MNAR}
& 5\%  & 1.2797 & 0.9583 & 0.8098 & 0.8267 & \textbf{0.3714} & 0.9811 & 0.6778 & 0.5275 \\
& 10\% & 1.5205 & 1.2191 & 1.0029 & 1.0528 & \textbf{0.5346} & 1.1291 & 0.9207 & 0.6726 \\
& 15\% & 1.6515 & 1.3775 & 1.1158 & 1.2033 & \textbf{0.6719} & 1.3409 & - & 0.7603 \\
& 20\% & 1.7428 & 1.4890 & 1.2036 & 1.3278 & \textbf{0.8280} & 1.4611 & - & 0.8326 \\
& 25\% & 1.8139 & 1.5642 & 1.2785 & 1.4375 & 1.0000 & 1.7106 & - & \textbf{0.8879} \\
& 30\% & 1.8735 & 1.6562 & 1.3512 & 1.5362 & 1.1631 & 1.8700 & - & \textbf{0.9610} \\
& 35\% & 1.9286 & 1.7141 & 1.4256 & 1.6301 & 1.3215 & 1.9737 & - & \textbf{1.0645} \\
& 40\% & 1.9839 & 1.7741 & 1.5117 & 1.7305 & 1.4727 & 1.9798 & - & \textbf{1.2065} \\
& 45\% & 2.0497 & 1.9019 & 1.6345 & 1.8460 & 1.6731 & 2.0857 & - & \textbf{1.3602}\\
\bottomrule
\end{tabular}

\caption{Comparison of imputation methods on the energy dataset under different missing mechanisms and missing rates using Wasserstein distance. Best results are shown in bold, - means the method either collapses or the value is too big to be meaningful.}
\label{tab:energy:wass}
\end{table}

\begin{table}[htbp]
\centering
\small
\setlength{\tabcolsep}{4pt}

\begin{tabular}{c|c|cccccccc}
\toprule
\multicolumn{10}{c}{\textbf{RMSE}} \\
\midrule
\textbf{Mechanism} & \textbf{Rate} & Mean & PCA & KNN & MICE & MissForest & GAIN & VAE & Proposal \\
\midrule

\multirow{9}{*}{MCAR}
& 5\%  & 1.0026 & 0.2797 & 0.0638 & 0.1573 & \textbf{0.0180} & 0.3061 & - & 0.0436 \\
& 10\% & 1.0018 & 0.2849 & 0.0644 & 0.1689 & \textbf{0.0180} & 0.3235 & - & 0.0447 \\
& 15\% & 1.0023 & 0.2908 & 0.0648 & 0.1832 & \textbf{0.0226} & 0.2830 & - & 0.0448 \\
& 20\% & 1.0028 & 0.2996 & 0.0655 & 0.2024 & \textbf{0.0265} & 0.3251 & - & 0.0448 \\
& 25\% & 1.0031 & 0.3272 & 0.0663 & 0.2309 & \textbf{0.0325} & 0.3220 & - & 0.0457 \\
& 30\% & 1.0029 & 0.3464 & 0.0675 & 0.2796 & \textbf{0.0462} & 0.3893 & - & 0.0466 \\
& 35\% & 1.0035 & 0.3831 & 0.0699 & 0.3610 & 0.0832 & 0.5010 & - & \textbf{0.0500} \\
& 40\% & 1.0043 & 0.4580 & 0.0783 & 0.4995 & 0.1797 & 0.5872 & - & \textbf{0.0590} \\
& 45\% & 1.0046 & 0.6664 & \textbf{0.1967} & 0.7925 & 0.4584 & 0.7551 & - & 0.2578 \\
\midrule

\multirow{9}{*}{MAR}
& 5\%  & 1.0930 & 0.2895 & 0.0686 & 0.1645 & \textbf{0.0197} & 0.3256 & - & 0.0455 \\
& 10\% & 1.0607 & 0.2941 & 0.0676 & 0.1751 & \textbf{0.0214} & 0.2934 & - & 0.0457 \\
& 15\% & 1.0401 & 0.2990 & 0.0670 & 0.1867 & \textbf{0.0227} & 0.3123 & - & 0.0462 \\
& 20\% & 1.0261 & 0.3061 & 0.0669 & 0.2044 & \textbf{0.0272} & 0.3235 & - & 0.0463 \\
& 25\% & 1.0143 & 0.3313 & 0.0671 & 0.2295 & \textbf{0.0301} & 0.4288 & - & 0.0464 \\
& 30\% & 1.0060 & 0.3474 & 0.0679 & 0.2711 & \textbf{0.0428} & 0.5561 & - & 0.0473 \\
& 35\% & 1.0013 & 0.3778 & 0.0699 & 0.3436 & 0.0712 & 0.7879 & - & \textbf{0.0490} \\
& 40\% & 0.9994 & 0.4395 & 0.0790 & 0.4939 & 0.1488 & 0.9756 & - & \textbf{0.0631} \\
& 45\% & 1.0001 & 0.6413 & \textbf{0.1915} & 0.7709 & 0.4008 & 1.0012 & - & 0.2687 \\
\midrule

\multirow{9}{*}{MNAR}
& 5\%  & 1.5296 & 0.3491 & 0.0854 & 0.2208 & \textbf{0.0270} & 0.3640 & - & 0.0551 \\
& 10\% & 1.3649 & 0.3407 & 0.0797 & 0.2684 & \textbf{0.0305} & 0.3384 & - & 0.0517 \\
& 15\% & 1.2548 & 0.3377 & 0.0769 & 0.3285 & \textbf{0.0358} & 0.4887 & - & 0.0507 \\
& 20\% & 1.1761 & 0.3677 & 0.0752 & 0.3909 & \textbf{0.0458} & 0.7647 & - & 0.0501 \\
& 25\% & 1.1178 & 0.3838 & 0.0745 & 0.4687 & 0.0573 & 1.0368 & - & \textbf{0.0505} \\
& 30\% & 1.0741 & 0.4341 & 0.0758 & 0.5547 & 0.1026 & 1.1490 & - & \textbf{0.0518} \\
& 35\% & 1.0426 & 0.5275 & 0.0895 & 0.6487 & 0.1797 & 1.2473 & - & \textbf{0.0701} \\
& 40\% & 1.0211 & 0.6778 & \textbf{0.1595} & 0.7605 & 0.3334 & 1.4548 & - & 0.1792 \\
& 45\% & 1.0085 & 0.7938 & 0.3852 & 0.8791 & 0.6144 & 1.5599 & - & \textbf{0.4960}\\
\bottomrule
\end{tabular}

\caption{Comparison of imputation methods on the finance dataset under different missing mechanisms and missing rates using RMSE. Best results are shown in bold, - means the method either collapses or the value is too big to be meaningful.}
\label{tab:finance:rmse}
\end{table}

\begin{table}[htbp]
\centering
\small
\setlength{\tabcolsep}{4pt}

\begin{tabular}{c|c|cccccccc}
\toprule
\multicolumn{10}{c}{\textbf{Wasserstein Distance}} \\
\midrule
\textbf{Mechanism} & \textbf{Rate} & Mean & PCA & KNN & MICE & MissForest & GAIN & VAE & Proposal \\
\midrule

\multirow{9}{*}{MCAR}
& 5\%  & 1.4133 & 0.7548 & 0.3406 & 0.5601 & \textbf{0.1514} & 0.7599 & - & 0.2814 \\
& 10\% & 1.6862 & 0.9278 & 0.4348 & 0.7150 & \textbf{0.1945} & 0.9711 & - & 0.3498 \\
& 15\% & 1.8629 & 1.0480 & 0.5009 & 0.8249 & \textbf{0.2303} & 1.0207 & - & 0.4046 \\
& 20\% & 1.9976 & 1.1416 & 0.5458 & 0.9312 & \textbf{0.2628} & 1.1750 & - & 0.4399 \\
& 25\% & 2.1192 & 1.2506 & 0.5786 & 1.0468 & \textbf{0.3085} & 1.2373 & - & 0.4701 \\
& 30\% & 2.2138 & 1.3296 & 0.6083 & 1.1820 & \textbf{0.3776} & 1.4173 & - & 0.5030 \\
& 35\% & 2.2928 & 1.4203 & 0.6391 & 1.3675 & \textbf{0.4879} & 1.6644 & - & 0.5280 \\
& 40\% & 2.3687 & 1.5475 & 0.6674 & 1.6351 & 0.7049 & 1.8487 & - & \textbf{0.5531} \\
& 45\% & 2.4364 & 1.8855 & 0.8231 & 2.1138 & 1.1885 & 2.0928 & - & \textbf{0.6361} \\
\midrule

\multirow{9}{*}{MAR}
& 5\%  & 1.4418 & 0.7706 & 0.3480 & 0.5741 & \textbf{0.1563} & 0.7965 & - & 0.2858 \\
& 10\% & 1.7038 & 0.9482 & 0.4421 & 0.7280 & \textbf{0.2018} & 0.9317 & - & 0.3534 \\
& 15\% & 1.8759 & 1.0625 & 0.5042 & 0.8387 & \textbf{0.2348} & 1.0743 & - & 0.4144 \\
& 20\% & 2.0088 & 1.1561 & 0.5514 & 0.9411 & \textbf{0.2683} & 1.1793 & - & 0.4494 \\
& 25\% & 2.1154 & 1.2614 & 0.5851 & 1.0474 & \textbf{0.3082} & 1.4321 & - & 0.4804 \\
& 30\% & 2.2131 & 1.3347 & 0.6141 & 1.1750 & \textbf{0.3778} & 1.6971 & - & 0.5040 \\
& 35\% & 2.2995 & 1.4165 & 0.6410 & 1.3502 & \textbf{0.4696} & 2.0899 & - & 0.5293 \\
& 40\% & 2.3632 & 1.5359 & 0.6748 & 1.6421 & 0.6587 & 2.3605 & - & \textbf{0.5510} \\
& 45\% & 2.4309 & 1.8716 & 0.8408 & 2.1037 & 1.1394 & 2.4080 & - & \textbf{0.6739} \\
\midrule

\multirow{9}{*}{MNAR}
& 5\%  & 1.6059 & 0.8317 & 0.3824 & 0.6333 & \textbf{0.1740} & 0.8452 & - & 0.3053 \\
& 10\% & 1.8350 & 1.0054 & 0.4781 & 0.8212 & \textbf{0.2267} & 1.0002 & - & 0.3757 \\
& 15\% & 1.9876 & 1.1174 & 0.5322 & 0.9920 & \textbf{0.2673} & 1.2299 & - & 0.4357 \\
& 20\% & 2.0931 & 1.2426 & 0.5720 & 1.1397 & \textbf{0.3197} & 1.5459 & - & 0.4630 \\
& 25\% & 2.1723 & 1.3162 & 0.5988 & 1.2931 & \textbf{0.3806} & 1.8797 & - & 0.4951 \\
& 30\% & 2.2552 & 1.4013 & 0.6269 & 1.4768 & \textbf{0.4807} & 2.1364 & - & 0.5184 \\
& 35\% & 2.3278 & 1.5281 & 0.6616 & 1.6672 & 0.6207 & 2.3184 & - & \textbf{0.5386} \\
& 40\% & 2.3970 & 1.7630 & 0.7400 & 1.9127 & 0.9300 & 2.6804 & - & \textbf{0.6140} \\
& 45\% & 2.4605 & 2.0131 & 1.0863 & 2.1703 & 1.4542 & 2.9189 & - & \textbf{0.9321}\\
\bottomrule
\end{tabular}

\caption{Comparison of imputation methods on the finance dataset under different missing mechanisms and missing rates using Wasserstein distance. Best results are shown in bold, - means the method either collapses or the value is too big to be meaningful.}
\label{tab:finance:wass}
\end{table}

\begin{table}[htbp]
\centering
\small
\setlength{\tabcolsep}{4pt}

\begin{tabular}{c|c|cccccccc}
\toprule
\multicolumn{10}{c}{\textbf{RMSE}} \\
\midrule
\textbf{Mechanism} & \textbf{Rate} & Mean & PCA & KNN & MICE & MissForest & GAIN & VAE & Proposal \\
\midrule

\multirow{9}{*}{MCAR}
& 5\%  & 1.0052 & 1.1198 & 0.6375 & 0.6747 & \textbf{0.5665} & 1.2295 & - & 0.8580 \\
& 10\% & 1.0122 & 1.0465 & \textbf{0.6941} & 0.7282 & 0.6808 & 1.0113 & 2.0285 & 0.9230 \\
& 15\% & 1.0047 & 1.0047 & \textbf{0.7267} & 0.7557 & 0.7457 & 1.6948 & - & 1.0318 \\
& 20\% & 1.0043 & 1.0043 & \textbf{0.7656} & 0.7907 & 0.8265 & 0.9247 & - & 1.0819 \\
& 25\% & 1.0010 & 1.0010 & \textbf{0.7964} & 0.8181 & 0.9323 & 0.9752 & - & 1.1267 \\
& 30\% & 1.0047 & 1.0047 & \textbf{0.8390} & 0.8566 & 0.9632 & 1.0106 & - & 1.1802 \\
& 35\% & 1.0087 & 1.0087 & \textbf{0.8839} & 0.8984 & 1.0247 & 1.0842 & - & 1.2428 \\
& 40\% & 1.0098 & 1.0098 & \textbf{0.9245} & 0.9330 & 0.9931 & 1.1633 & - & 1.3072 \\
& 45\% & 1.0072 & 1.0072 & \textbf{0.9652} & 0.9688 & 1.0425 & 1.6922 & - & 1.3875 \\
\midrule

\multirow{9}{*}{MAR}
& 5\%  & 1.0064 & 1.0818 & 0.6257 & 0.6725 & \textbf{0.5099} & 1.2243 & 1.6748 & 0.8511 \\
& 10\% & 0.9922 & 1.0225 & 0.6592 & 0.6984 & \textbf{0.6284} & 0.8972 & - & 0.9041 \\
& 15\% & 0.9917 & 0.9917 & \textbf{0.6962} & 0.7288 & 0.7118 & 1.1041 & - & 0.9689 \\
& 20\% & 0.9892 & 0.9892 & \textbf{0.7353} & 0.7671 & 0.7737 & 1.2540 & 1.1667 & 1.0380 \\
& 25\% & 0.9914 & 0.9914 & \textbf{0.7859} & 0.8111 & 0.8385 & 0.9091 & 1.3215 & 1.1107 \\
& 30\% & 0.9919 & 0.9919 & \textbf{0.8214} & 0.8463 & 0.9829 & 0.8755 & - & 1.1586 \\
& 35\% & 0.9962 & 0.9962 & \textbf{0.8677} & 0.8929 & 0.9581 & 0.9841 & - & 1.2403 \\
& 40\% & 1.0021 & 1.0021 & \textbf{0.9152} & 0.9330 & 1.3216 & 1.3011 & - & 1.3582 \\
& 45\% & 1.0028 & 1.0028 & \textbf{0.9569} & 0.9665 & 1.0520 & 1.0070 & - & 1.3558 \\
\midrule

\multirow{9}{*}{MNAR}
& 5\%  & 1.1651 & 1.1006 & 0.6975 & 0.7582 & \textbf{0.5321} & 1.3763 & - & 0.9268 \\
& 10\% & 1.1061 & 1.1061 & 0.6988 & 0.7799 & \textbf{0.6050} & 1.7733 & - & 0.9823 \\
& 15\% & 1.0637 & 1.0637 & 0.7209 & 0.8087 & \textbf{0.7000} & 0.8755 & - & 1.0316 \\
& 20\% & 1.0337 & 1.0337 & \textbf{0.7352} & 0.8252 & 0.8009 & 0.7711 & - & 1.0914 \\
& 25\% & 1.0085 & 1.0085 & \textbf{0.7532} & 0.8390 & 0.9642 & 0.8691 & - & 1.1234 \\
& 30\% & 0.9892 & 0.9892 & \textbf{0.7915} & 0.8558 & 1.0279 & 1.0071 & - & 1.1686 \\
& 35\% & 0.9795 & 0.9795 & \textbf{0.8280} & 0.8789 & 0.9576 & 1.0470 & - & 1.2794 \\
& 40\% & 0.9764 & 0.9764 & \textbf{0.8759} & 0.9049 & 1.1382 & 1.1427 & - & 1.3142 \\
& 45\% & 0.9848 & 0.9848 & \textbf{0.9333} & 0.9455 & 1.5230 & 1.4324 & - & 1.3551\\
\bottomrule
\end{tabular}

\caption{Comparison of imputation methods on the powerplant dataset under different missing mechanisms and missing rates using RMSE. Best results are shown in bold, - means the method either collapses or the value is too big to be meaningful.}
\label{tab:powerplant:rmse}
\end{table}

\begin{table}[htbp]
\centering
\small
\setlength{\tabcolsep}{4pt}

\begin{tabular}{c|c|cccccccc}
\toprule
\multicolumn{10}{c}{\textbf{Wasserstein Distance}} \\
\midrule
\textbf{Mechanism} & \textbf{Rate} & Mean & PCA & KNN & MICE & MissForest & GAIN & VAE & Proposal \\
\midrule

\multirow{9}{*}{MCAR}
& 5\%  & 0.4808 & 0.4107 & 0.3396 & 0.3490 & \textbf{0.2841} & 0.5069 & - & 0.3059 \\
& 10\% & 0.6583 & 0.5339 & 0.4758 & 0.4869 & \textbf{0.3829} & 0.5890 & 0.5718 & 0.3859 \\
& 15\% & 0.7847 & 0.7846 & 0.5696 & 0.5859 & 0.4622 & 1.0361 & - & \textbf{0.4524} \\
& 20\% & 0.8858 & 0.8858 & 0.6543 & 0.6768 & 0.5291 & 0.7962 & - & \textbf{0.4883} \\
& 25\% & 0.9740 & 0.9739 & 0.7363 & 0.7621 & 0.6188 & 0.8186 & - & \textbf{0.5131} \\
& 30\% & 1.0587 & 1.0587 & 0.8331 & 0.8598 & 0.7070 & 1.0629 & - & \textbf{0.5322} \\
& 35\% & 1.1408 & 1.1409 & 0.9485 & 0.9703 & 0.8313 & 1.0955 & - & \textbf{0.5515} \\
& 40\% & 1.2204 & 1.2204 & 1.0799 & 1.0967 & 0.9865 & 1.2297 & - & \textbf{0.5533} \\
& 45\% & 1.2958 & 1.2958 & 1.2205 & 1.2316 & 1.2036 & 1.5909 & - & \textbf{0.5627} \\
\midrule

\multirow{9}{*}{MAR}
& 5\%  & 0.5038 & 0.4257 & 0.3478 & 0.3577 & \textbf{0.2831} & 0.5226 & 0.4705 & 0.3143 \\
& 10\% & 0.6787 & 0.5633 & 0.4776 & 0.4899 & \textbf{0.3915} & 0.5523 & - & 0.4038 \\
& 15\% & 0.8026 & 0.8022 & 0.5747 & 0.5918 & 0.4749 & 0.7033 & - & \textbf{0.4733} \\
& 20\% & 0.8987 & 0.8987 & 0.6571 & 0.6782 & 0.5303 & 0.9871 & 0.7535 & \textbf{0.5259} \\
& 25\% & 0.9882 & 0.9882 & 0.7499 & 0.7767 & 0.6165 & 0.8260 & 0.8262 & \textbf{0.5627} \\
& 30\% & 1.0647 & 1.0646 & 0.8405 & 0.8742 & 0.7726 & 0.8872 & - & \textbf{0.5813} \\
& 35\% & 1.1408 & 1.1408 & 0.9477 & 0.9807 & 0.8654 & 1.0096 & - & \textbf{0.6026} \\
& 40\% & 1.2169 & 1.2169 & 1.0737 & 1.1005 & 1.1684 & 1.3991 & - & \textbf{0.5862} \\
& 45\% & 1.2937 & 1.2937 & 1.2170 & 1.2295 & 1.2013 & 1.1880 & - & \textbf{0.5620} \\
\midrule

\multirow{9}{*}{MNAR}
& 5\%  & 0.5506 & 0.4201 & 0.3849 & 0.3956 & \textbf{0.3036} & 0.5825 & - & 0.3476 \\
& 10\% & 0.7338 & 0.7350 & 0.5222 & 0.5492 & \textbf{0.4243} & 0.8338 & - & 0.4712 \\
& 15\% & 0.8564 & 0.8568 & 0.6206 & 0.6626 & \textbf{0.5143} & 0.7012 & - & 0.5466 \\
& 20\% & 0.9512 & 0.9516 & 0.7031 & 0.7587 & 0.5991 & 0.7317 & - & \textbf{0.5936} \\
& 25\% & 1.0249 & 1.0249 & 0.7762 & 0.8418 & 0.7113 & 0.8411 & - & \textbf{0.6405} \\
& 30\% & 1.0905 & 1.0904 & 0.8575 & 0.9233 & 0.7972 & 0.9378 & - & \textbf{0.6355} \\
& 35\% & 1.1509 & 1.1510 & 0.9483 & 1.0055 & 0.7939 & 0.9695 & - & \textbf{0.6514} \\
& 40\% & 1.2175 & 1.2176 & 1.0653 & 1.1013 & 1.0889 & 1.1205 & - & \textbf{0.6438} \\
& 45\% & 1.2894 & 1.2894 & 1.2050 & 1.2202 & 1.4672 & 1.5267 & - & \textbf{0.6364}\\
\bottomrule
\end{tabular}

\caption{Comparison of imputation methods on the powerplant dataset under different missing mechanisms and missing rates using Wasserstein distance. Best results are shown in bold, - means the method either collapses or the value is too big to be meaningful.}
\label{tab:powerplant:wass}
\end{table}

\begin{table}[htbp]
\centering
\small
\setlength{\tabcolsep}{4pt}

\begin{tabular}{c|c|cccccccc}
\toprule
\multicolumn{10}{c}{\textbf{RMSE}} \\
\midrule
\textbf{Mechanism} & \textbf{Rate} & Mean & PCA & KNN & MICE & MissForest & GAIN & VAE & Proposal \\
\midrule

\multirow{9}{*}{MCAR}
& 5\%  & 0.9679 & 1.6834 & 0.7992 & 0.7005 & \textbf{0.6382} & 0.8832 & - & 1.0436 \\
& 10\% & 0.9814 & 1.3382 & 0.8170 & 0.7573 & \textbf{0.7165} & 0.9044 & - & 1.1017 \\
& 15\% & 0.9959 & 1.1436 & 0.8381 & 0.8011 & \textbf{0.7957} & 0.9267 & - & 1.1121 \\
& 20\% & 1.0121 & 1.0556 & 0.8694 & \textbf{0.8568} & 0.8698 & 0.9364 & - & 1.1852 \\
& 25\% & 1.0158 & 1.0142 & \textbf{0.8914} & 0.8943 & 0.9242 & 0.9455 & - & 1.2393 \\
& 30\% & 1.0087 & 0.9875 & \textbf{0.8993} & 0.9130 & 0.9601 & 1.0118 & - & 1.3380 \\
& 35\% & 1.0016 & 0.9789 & \textbf{0.9141} & 0.9359 & 1.0121 & 0.9868 & - & 1.3696 \\
& 40\% & 1.0022 & 0.9806 & \textbf{0.9385} & 0.9594 & 1.0572 & 1.0060 & - & 1.7441 \\
& 45\% & 1.0021 & 0.9906 & \textbf{0.9676} & 0.9790 & 1.0729 & 1.0838 & - & 1.4415 \\
\midrule

\multirow{9}{*}{MAR}
& 5\%  & 1.0289 & 1.6952 & 0.8580 & 0.7590 & \textbf{0.7169} & 0.9427 & - & 1.1011 \\
& 10\% & 1.0148 & 1.3366 & 0.8491 & 0.7925 & \textbf{0.7581} & 0.9087 & - & 1.1106 \\
& 15\% & 1.0089 & 1.1503 & 0.8519 & 0.8135 & \textbf{0.8028} & 0.9109 & - & 1.1542 \\
& 20\% & 0.9994 & 1.0353 & 0.8487 & \textbf{0.8401} & 0.8478 & 0.9950 & - & 1.1967 \\
& 25\% & 1.0012 & 0.9792 & \textbf{0.8615} & 0.8674 & 0.8944 & 1.1225 & - & 1.2196 \\
& 30\% & 1.0042 & 0.9695 & \textbf{0.8866} & 0.9012 & 0.9284 & 1.1205 & - & 1.2732 \\
& 35\% & 1.0026 & 0.9748 & \textbf{0.9088} & 0.9307 & 0.9814 & 1.1039 & - & 1.3148 \\
& 40\% & 0.9940 & 0.9705 & \textbf{0.9254} & 0.9481 & 1.0345 & 1.0165 & - & 1.4154 \\
& 45\% & 0.9967 & 0.9862 & \textbf{0.9603} & 0.9762 & 1.0490 & 1.0638 & - & 1.4266 \\
\midrule

\multirow{9}{*}{MNAR}
& 5\%  & 1.6981 & 1.8982 & 1.5098 & 1.4353 & \textbf{1.3438} & 1.4768 & - & 1.5586 \\
& 10\% & 1.4313 & 1.5472 & 1.2617 & 1.2370 & \textbf{1.1585} & 1.3075 & - & 1.3816 \\
& 15\% & 1.2831 & 1.3441 & 1.1306 & 1.1493 & \textbf{1.1107} & 1.2247 & - & 1.3162 \\
& 20\% & 1.1889 & 1.2082 & \textbf{1.0504} & 1.1070 & 1.0974 & 1.1287 & - & 1.2951 \\
& 25\% & 1.1221 & 1.1203 & \textbf{1.0009} & 1.0835 & 1.0982 & 1.0769 & - & 1.3004 \\
& 30\% & 1.0757 & 1.0589 & \textbf{0.9713} & 1.0678 & 1.0935 & 1.0232 & - & 1.3111 \\
& 35\% & 1.0417 & 1.0222 & \textbf{0.9575} & 1.0558 & 1.1103 & 1.1998 & - & 1.3628 \\
& 40\% & 1.0165 & 0.9967 & \textbf{0.9514} & 1.0358 & 1.1125 & 1.1326 & - & 1.4132 \\
& 45\% & 1.0047 & 1.6405 & \textbf{0.9655} & 1.0050 & 1.1477 & 1.2245 & - & 1.3810\\
\bottomrule
\end{tabular}

\caption{Comparison of imputation methods on the wine dataset under different missing mechanisms and missing rates using RMSE. Best results are shown in bold, - means the method either collapses or the value is too big to be meaningful.}
\label{tab:wine:rmse}
\end{table}

\begin{table}[htbp]
\centering
\small
\setlength{\tabcolsep}{4pt}

\begin{tabular}{c|c|cccccccc}
\toprule
\multicolumn{10}{c}{\textbf{Wasserstein Distance}} \\
\midrule
\textbf{Mechanism} & \textbf{Rate} & Mean & PCA & KNN & MICE & MissForest & GAIN & VAE & Proposal \\
\midrule

\multirow{9}{*}{MCAR}
& 5\%  & 0.8077 & 0.9515 & 0.7224 & 0.6592 & \textbf{0.6069} & 0.7613 & - & 0.7772 \\
& 10\% & 1.0463 & 1.1133 & 0.9422 & 0.8905 & \textbf{0.8380} & 0.9933 & - & 1.0047 \\
& 15\% & 1.1840 & 1.1723 & 1.0737 & 1.0352 & \textbf{0.9882} & 1.1221 & - & 1.1181 \\
& 20\% & 1.2776 & 1.2165 & 1.1608 & 1.1356 & \textbf{1.0780} & 1.2087 & - & 1.2030 \\
& 25\% & 1.3530 & 1.2679 & 1.2304 & 1.2178 & \textbf{1.1491} & 1.2679 & - & 1.2540 \\
& 30\% & 1.4192 & 1.3348 & 1.2882 & 1.2863 & \textbf{1.1937} & 1.3602 & - & 1.3070 \\
& 35\% & 1.4859 & 1.4111 & 1.3487 & 1.3664 & \textbf{1.2493} & 1.3743 & - & 1.3362 \\
& 40\% & 1.5633 & 1.5034 & 1.4300 & 1.4669 & \textbf{1.3104} & 1.5175 & - & 1.3968 \\
& 45\% & 1.6540 & 1.6204 & 1.5538 & 1.5954 & 1.4767 & 1.6584 & - & \textbf{1.3768} \\
\midrule

\multirow{9}{*}{MAR}
& 5\%  & 0.8293 & 0.9591 & 0.7397 & 0.6798 & \textbf{0.6428} & 0.7764 & - & 0.7925 \\
& 10\% & 1.0627 & 1.1027 & 0.9554 & 0.9073 & \textbf{0.8536} & 0.9818 & - & 1.0021 \\
& 15\% & 1.1914 & 1.1710 & 1.0771 & 1.0409 & \textbf{0.9802} & 1.1061 & - & 1.1239 \\
& 20\% & 1.2889 & 1.2161 & 1.1638 & 1.1420 & \textbf{1.0770} & 1.2492 & - & 1.2083 \\
& 25\% & 1.3684 & 1.2724 & 1.2327 & 1.2213 & \textbf{1.1491} & 1.3041 & - & 1.2615 \\
& 30\% & 1.4346 & 1.3385 & 1.2950 & 1.2967 & \textbf{1.1997} & 1.4280 & - & 1.3040 \\
& 35\% & 1.5011 & 1.4222 & 1.3613 & 1.3806 & \textbf{1.2549} & 1.4123 & - & 1.3272 \\
& 40\% & 1.5691 & 1.5068 & 1.4365 & 1.4734 & \textbf{1.3206} & 1.4902 & - & 1.3834 \\
& 45\% & 1.6556 & 1.6219 & 1.5589 & 1.6006 & 1.4711 & 1.5770 & - & \textbf{1.3812} \\
\midrule

\multirow{9}{*}{MNAR}
& 5\%  & 1.0044 & 0.9647 & 0.9218 & 0.8631 & \textbf{0.8058} & 0.9108 & - & 0.8879 \\
& 10\% & 1.2026 & 1.1436 & 1.1105 & 1.0672 & \textbf{1.0054} & 1.1298 & - & 1.0668 \\
& 15\% & 1.3140 & 1.2367 & 1.2128 & 1.1934 & \textbf{1.1448} & 1.2459 & - & 1.1675 \\
& 20\% & 1.3893 & 1.3089 & 1.2803 & 1.2907 & 1.2288 & 1.3220 & - & \textbf{1.2210} \\
& 25\% & 1.4463 & 1.3628 & 1.3296 & 1.3681 & 1.2964 & 1.3604 & - & \textbf{1.2636} \\
& 30\% & 1.4977 & 1.4195 & 1.3720 & 1.4406 & 1.3489 & 1.4271 & - & \textbf{1.3043} \\
& 35\% & 1.5465 & 1.4790 & 1.4164 & 1.5089 & 1.4050 & 1.5566 & - & \textbf{1.3440} \\
& 40\% & 1.6030 & 1.5486 & 1.4725 & 1.5762 & 1.4440 & 1.5921 & - & \textbf{1.3870} \\
& 45\% & 1.6730 & \textbf{0.9928} & 1.5690 & 1.6398 & 1.5664 & 1.6800 & - & 1.4303\\
\bottomrule
\end{tabular}

\caption{Comparison of imputation methods on the wine dataset under different missing mechanisms and missing rates using Wasserstein distance. Best results are shown in bold, - means the method either collapses or the value is too big to be meaningful.}
\label{tab:wine:wass}
\end{table}

\begin{table}[htbp]
\centering
\small
\setlength{\tabcolsep}{4pt}

\begin{tabular}{c|c|cccccccc}
\toprule
\multicolumn{10}{c}{\textbf{RMSE}} \\
\midrule
\textbf{Mechanism} & \textbf{Rate} & Mean & PCA & KNN & MICE & MissForest & GAIN & VAE & Proposal \\
\midrule

\multirow{9}{*}{MCAR}
& 5\%  & 0.9973 & 0.1450 & 0.1005 & \textbf{0.0066} & 0.0094 & 0.0651 & 0.0736 & 0.0674 \\
& 10\% & 0.9981 & 0.1452 & 0.0999 & \textbf{0.0059} & 0.0110 & 0.0821 & 0.0881 & 0.0671 \\
& 15\% & 0.9984 & 0.1453 & 0.1002 & \textbf{0.0060} & 0.0125 & 0.0741 & - & 0.0673 \\
& 20\% & 0.9983 & 0.1455 & 0.1001 & \textbf{0.0061} & 0.0155 & 0.0769 & - & 0.0671 \\
& 25\% & 0.9981 & 0.1457 & 0.1004 & \textbf{0.0054} & 0.0194 & 0.0856 & - & 0.0677 \\
& 30\% & 0.9983 & 0.1461 & 0.1007 & \textbf{0.0078} & 0.0263 & 0.1956 & - & 0.0679 \\
& 35\% & 0.9984 & 0.1467 & 0.1011 & \textbf{0.0173} & 0.0383 & 0.1267 & - & 0.0679 \\
& 40\% & 0.9986 & 0.1479 & 0.1022 & \textbf{0.0447} & 0.0620 & 0.1433 & - & 0.0685 \\
& 45\% & 0.9987 & 0.1524 & 0.1063 & 0.9002 & 0.1393 & 0.3140 & - & \textbf{0.0735} \\
\midrule

\multirow{9}{*}{MAR}
& 5\%  & 1.0638 & 0.1482 & 0.1161 & \textbf{0.0067} & 0.0125 & 0.0689 & 0.0823 & 0.0709 \\
& 10\% & 1.0403 & 0.1467 & 0.1084 & \textbf{0.0060} & 0.0126 & 0.0726 & - & 0.0691 \\
& 15\% & 1.0252 & 0.1458 & 0.1046 & \textbf{0.0061} & 0.0139 & 0.0954 & - & 0.0676 \\
& 20\% & 1.0136 & 0.1458 & 0.1025 & \textbf{0.0062} & 0.0161 & 0.1697 & - & 0.0670 \\
& 25\% & 1.0047 & 0.1458 & 0.1011 & \textbf{0.0058} & 0.0199 & 0.2859 & - & 0.0667 \\
& 30\% & 0.9983 & 0.1460 & 0.1005 & \textbf{0.0071} & 0.0265 & 0.3470 & - & 0.0671 \\
& 35\% & 0.9947 & 0.1465 & 0.1004 & \textbf{0.0151} & 0.0375 & 0.6952 & - & 0.0670 \\
& 40\% & 0.9932 & 0.1474 & 0.1009 & \textbf{0.0511} & 0.0596 & 1.1034 & - & 0.0685 \\
& 45\% & 0.9942 & 0.1520 & 0.1055 & 0.8242 & 0.1270 & 1.1228 & - & \textbf{0.0736} \\
\midrule

\multirow{9}{*}{MNAR}
& 5\%  & 1.1436 & 0.1639 & 0.1150 & \textbf{0.0080} & 0.0132 & 0.0798 & 0.0751 & 0.0779 \\
& 10\% & 1.0926 & 0.1628 & 0.1126 & \textbf{0.0139} & 0.0395 & 0.1948 & 0.0919 & 0.0762 \\
& 15\% & 1.0550 & 0.1618 & 0.1110 & \textbf{0.0267} & 0.0593 & 0.3195 & 0.1316 & 0.0773 \\
& 20\% & 1.0266 & 0.1609 & 0.1098 & \textbf{0.0438} & 0.0834 & 0.4853 & - & 0.0761 \\
& 25\% & 1.0059 & 0.1601 & 0.1093 & \textbf{0.0681} & 0.1031 & 1.2109 & - & 0.0741 \\
& 30\% & 0.9921 & 0.1612 & 0.1097 & 0.1109 & 0.1270 & 1.2721 & - & \textbf{0.0744} \\
& 35\% & 0.9850 & 0.1730 & 0.1124 & 0.1823 & 0.1744 & 1.1063 & - & \textbf{0.0777} \\
& 40\% & 0.9837 & 0.2167 & 0.1277 & 0.2881 & 0.2461 & 1.4906 & - & \textbf{0.0945} \\
& 45\% & 0.9879 & 0.3512 & 0.1956 & 0.4636 & 0.3722 & 1.7851 & - & \textbf{0.1736}\\
\bottomrule
\end{tabular}

\caption{Comparison of imputation methods on the facial expression dataset under different missing mechanisms and missing rates using RMSE. Best results are shown in bold, - means the method either collapses or the value is too big to be meaningful.}
\label{tab:facialexpression:rmse}
\end{table}

\begin{table}[htbp]
\centering
\small
\setlength{\tabcolsep}{4pt}

\begin{tabular}{c|c|cccccccc}
\toprule
\multicolumn{10}{c}{\textbf{Wasserstein Distance}} \\
\midrule
\textbf{Mechanism} & \textbf{Rate} & Mean & PCA & KNN & MICE & MissForest & GAIN & VAE & Proposal \\
\midrule

\multirow{9}{*}{MCAR}
& 5\%  & 2.1107 & 0.7430 & 0.5838 & \textbf{0.1526} & 0.1674 & 0.5151 & 0.4689 & 0.5025 \\
& 10\% & 2.5143 & 0.8926 & 0.6981 & \textbf{0.1734} & 0.2191 & 0.6938 & 0.5846 & 0.6002 \\
& 15\% & 2.7719 & 1.0175 & 0.7757 & \textbf{0.1936} & 0.2645 & 0.7285 & - & 0.6666 \\
& 20\% & 2.9654 & 1.1191 & 0.8582 & \textbf{0.2105} & 0.3165 & 0.8044 & - & 0.7156 \\
& 25\% & 3.1244 & 1.2027 & 0.9359 & \textbf{0.2011} & 0.3779 & 0.9093 & - & 0.7724 \\
& 30\% & 3.2595 & 1.2575 & 1.0057 & \textbf{0.2432} & 0.4633 & 1.4685 & - & 0.8289 \\
& 35\% & 3.3775 & 1.3031 & 1.0632 & \textbf{0.3641} & 0.5860 & 1.2297 & - & 0.8711 \\
& 40\% & 3.4826 & 1.3393 & 1.0997 & \textbf{0.6715} & 0.8098 & 1.3634 & - & 0.9136 \\
& 45\% & 3.5810 & 1.3662 & 1.1326 & 3.3864 & 1.2355 & 1.9350 & - & \textbf{0.9495} \\
\midrule

\multirow{9}{*}{MAR}
& 5\%  & 2.1519 & 0.7441 & 0.5924 & \textbf{0.1535} & 0.1710 & 0.5227 & 0.4714 & 0.5074 \\
& 10\% & 2.5456 & 0.8971 & 0.7041 & \textbf{0.1741} & 0.2193 & 0.6470 & - & 0.6035 \\
& 15\% & 2.7936 & 1.0131 & 0.7875 & \textbf{0.1945} & 0.2650 & 0.7977 & - & 0.6654 \\
& 20\% & 2.9809 & 1.1186 & 0.8609 & \textbf{0.2114} & 0.3144 & 1.0547 & - & 0.7148 \\
& 25\% & 3.1307 & 1.2050 & 0.9388 & \textbf{0.2131} & 0.3766 & 1.6357 & - & 0.7660 \\
& 30\% & 3.2618 & 1.2602 & 1.0057 & \textbf{0.2387} & 0.4643 & 1.8959 & - & 0.8214 \\
& 35\% & 3.3766 & 1.3020 & 1.0570 & \textbf{0.3568} & 0.5886 & 2.7337 & - & 0.8702 \\
& 40\% & 3.4813 & 1.3419 & 1.1040 & \textbf{0.7643} & 0.8111 & 3.7308 & - & 0.9139 \\
& 45\% & 3.5791 & 1.3648 & 1.1375 & 3.2397 & 1.2195 & 3.7267 & - & \textbf{0.9465} \\
\midrule

\multirow{9}{*}{MNAR}
& 5\%  & 2.0964 & 0.7498 & 0.5938 & \textbf{0.1577} & 0.1790 & 0.5264 & 0.4653 & 0.5063 \\
& 10\% & 2.4793 & 0.9228 & 0.7199 & \textbf{0.1844} & 0.2600 & 0.7431 & 0.6005 & 0.6059 \\
& 15\% & 2.7317 & 1.0431 & 0.8250 & \textbf{0.2316} & 0.3530 & 1.0947 & 0.7633 & 0.6829 \\
& 20\% & 2.9193 & 1.1375 & 0.9061 & \textbf{0.3032} & 0.4774 & 1.6926 & - & 0.7435 \\
& 25\% & 3.0790 & 1.2120 & 0.9787 & \textbf{0.4461} & 0.6278 & 2.8809 & - & 0.7912 \\
& 30\% & 3.2175 & 1.2783 & 1.0379 & \textbf{0.6796} & 0.8128 & 3.1091 & - & 0.8407 \\
& 35\% & 3.3423 & 1.3363 & 1.0886 & 1.0062 & 1.0434 & 3.0137 & - & \textbf{0.8987} \\
& 40\% & 3.4571 & 1.4216 & 1.1421 & 1.3969 & 1.3196 & 3.8546 & - & \textbf{0.9427} \\
& 45\% & 3.5643 & 1.6991 & 1.2442 & 2.0484 & 1.7205 & 4.4749 & - & \textbf{1.0328}\\
\bottomrule
\end{tabular}

\caption{Comparison of imputation methods on the facial expression dataset under different missing mechanisms and missing rates using Wasserstein distance. Best results are shown in bold, - means the method either collapses or the value is too big to be meaningful.}
\label{tab:facialexpression:wass}
\end{table}

\begin{table}[htbp]
\centering
\small
\setlength{\tabcolsep}{4pt}

\begin{tabular}{c|c|cccccccc}
\toprule
\multicolumn{10}{c}{\textbf{Prediction RMSE}} \\
\midrule
\textbf{Mechanism} & \textbf{Rate} & Mean & PCA & KNN & MICE & MissForest & GAIN & VAE & Proposal \\
\midrule

\multirow{1}{*}{Complete}
& 0\%  & 0.2661 & 0.2661 & 0.2661 & 0.2661 & 0.2661 & 0.2661 & 0.2661 & 0.2661 \\
\midrule

\multirow{9}{*}{MCAR}
& 5\%  & 0.3155 & 0.2806 & 0.2846 & 0.2667 & \textbf{0.2665} & 0.2822 & 0.2884 & 0.2748 \\
& 10\% & 0.3526 & 0.2881 & 0.2911 & 0.2709 & \textbf{0.2672} & 0.2893 & 0.2970 & 0.2806 \\
& 15\% & 0.4133 & 0.2970 & 0.2951 & 0.2707 & \textbf{0.2677} & 0.2950 & 0.3047 & 0.2827 \\
& 20\% & 0.5345 & 0.3110 & 0.2983 & 0.2735 & \textbf{0.2681} & 0.3032 & 0.3097 & 0.2824 \\
& 25\% & 0.7718 & 0.3336 & 0.3023 & 0.2804 & \textbf{0.2699} & 0.3084 & 0.3195 & 0.2842 \\
& 30\% & 1.2563 & 0.3732 & 0.3058 & 0.3139 & \textbf{0.2728} & 0.3339 & 0.3283 & 0.2846 \\
& 35\% & 2.3976 & 0.4795 & 0.3071 & 0.3391 & \textbf{0.2764} & 0.3584 & 0.3480 & 0.2829 \\
& 40\% & 5.5979 & 0.9595 & 0.3142 & 0.5147 & 0.2843 & 0.4152 & 0.3826 & \textbf{0.2815} \\
& 45\% & 19.8897 & 6.3432 & 0.3142 & 2.2280 & 0.3178 & 0.5523 & 0.4805 & \textbf{0.2800} \\
\midrule

\multirow{9}{*}{MAR}
& 5\%  & 0.3227 & 0.2816 & 0.2843 & 0.2666 & \textbf{0.2663} & 0.2810 & 0.2879 & 0.2751 \\
& 10\% & 0.3714 & 0.2912 & 0.2905 & 0.2718 & \textbf{0.2664} & 0.2882 & 0.2975 & 0.2799 \\
& 15\% & 0.4603 & 0.3051 & 0.2944 & 0.2730 & \textbf{0.2675} & 0.2960 & 0.3035 & 0.2816 \\
& 20\% & 0.6259 & 0.3263 & 0.2983 & 0.2783 & \textbf{0.2685} & 0.3082 & 0.3130 & 0.2818 \\
& 25\% & 0.9044 & 0.3473 & 0.3016 & 0.2961 & \textbf{0.2693} & 0.3222 & 0.3198 & 0.2835 \\
& 30\% & 1.4504 & 0.3954 & 0.3041 & 0.3202 & \textbf{0.2722} & 0.3678 & 0.3297 & 0.2825 \\
& 35\% & 2.6638 & 0.4739 & 0.3065 & 0.3618 & \textbf{0.2749} & 0.4505 & 0.3474 & 0.2825 \\
& 40\% & 5.7330 & 0.7161 & 0.3122 & 0.7238 & \textbf{0.2804} & 0.6194 & 0.3885 & 0.2811 \\
& 45\% & 17.7526 & 2.8746 & 0.3246 & 2.6903 & 0.3093 & 0.6543 & 0.4923 & \textbf{0.2833} \\
\midrule

\multirow{9}{*}{MNAR}
& 5\%  & 0.3526 & 0.2876 & 0.2883 & 0.2703 & \textbf{0.2681} & 0.2846 & 0.2910 & 0.2764 \\
& 10\% & 0.4495 & 0.2981 & 0.2931 & 0.2756 & \textbf{0.2690} & 0.2996 & 0.3021 & 0.2783 \\
& 15\% & 0.6217 & 0.3102 & 0.2975 & 0.2815 & \textbf{0.2691} & 0.3201 & 0.3118 & 0.2793 \\
& 20\% & 0.8702 & 0.3316 & 0.3001 & 0.2975 & \textbf{0.2720} & 0.3505 & 0.3171 & 0.2808 \\
& 25\% & 1.3047 & 0.3808 & 0.3028 & 0.3170 & \textbf{0.2774} & 0.5281 & 0.3291 & 0.2810 \\
& 30\% & 2.0948 & 0.4412 & 0.3065 & 0.3879 & 0.2824 & 0.5836 & 0.3421 & \textbf{0.2807} \\
& 35\% & 3.6230 & 0.5280 & 0.3111 & 0.5290 & 0.2874 & 0.8792 & 1.0014 & \textbf{0.2813} \\
& 40\% & 7.3746 & 0.6429 & 0.3188 & 0.9523 & 0.2997 & 2.9365 & 1.0014 & \textbf{0.2813} \\
& 45\% & 20.4163 & 6.5636 & 0.3191 & 3.0217 & 0.3631 & 0.9880 & 0.5243 & \textbf{0.2876}\\
\bottomrule
\end{tabular}

\caption{Comparison of imputation methods on the superconductivity dataset under different missing mechanisms and missing rates. Performance is evaluated using RMSE on a downstream regression task. Best results are shown in bold, - means the method either collapses or the value is too big to be meaningful.}
\label{tab:superconductivity:pred}
\end{table}

\end{document}